\documentclass[letterpaper,10pt]{article}
\usepackage{fullpage}
\usepackage{amsmath,amsthm}
\usepackage{times}
\usepackage{soul}

\usepackage[colorlinks=true, citecolor=blue]{hyperref}
\usepackage{url}
\usepackage[utf8]{inputenc}
\usepackage[small]{caption}
\usepackage{graphicx}
\usepackage{amsmath}
\usepackage{amsthm}
\usepackage{booktabs}
\usepackage[maxnames=10,backref=true,style=alphabetic]{biblatex}
\newcommand{\citet}[1]{\citeauthor*{#1}~\cite{#1}}
\usepackage{color}
\usepackage{amsfonts,amssymb,bbm,bm}
\usepackage{empheq}
\usepackage{cleveref}
\usepackage{xcolor}
\usepackage{tcolorbox}
\usepackage{tikz}
\usetikzlibrary{shadows,arrows.meta,positioning,backgrounds,fit,chains,scopes}
\usepackage{caption}
\usepackage{subcaption}
\usepackage{nicematrix}
\usepackage[parfill]{parskip}
\usepackage[colorinlistoftodos]{todonotes}

\usepackage{algorithmic}
\usepackage[noend,ruled]{algorithm2e}

\SetCommentSty{mycommfont}
\SetKwInput{KwInput}{Input}
\SetKwInput{KwOutput}{Output}  
\SetAlFnt{\small}
\SetAlCapFnt{\small}
\SetAlCapNameFnt{\small}
\SetAlCapHSkip{0pt}
\IncMargin{-\parindent}

\newtheorem{theorem}{Theorem}

\newcommand{\argmax}{\mathrm{argmax}}
\newcommand{\argmin}{\mathrm{argmin}}
\newcommand{\one}{\mathbbm{1}}
\newcommand{\set}[1]{\left\{ #1 \right\}}
\newcommand{\abs}[1]{\left| #1 \right|}
\newcommand{\norm}[1]{\left \lVert #1 \right \rVert}
\renewcommand{\bar}{\overline}
\renewcommand{\tilde}{\widetilde}
\newcommand{\E}{\mathbb{E}}
\newcommand{\calL}{\mathcal{L}}

\newcommand{\calS}{\mathcal{S}}

\newcommand{\calQ}{\mathcal{Q}}

\newcommand{\calC}{\mathcal{C}}

\newcommand{\calM}{\mathcal{M}}
\newcommand{\calE}{\mathcal{E}}

\newcommand{\calA}{\mathcal{A}}

\newcommand{\R}{\mathbb{R}}

\newcommand{\reg}{\textrm{Reg}}

\newcommand{\tpi}{\tilde{\pi}}
\newtheorem{lemma}{Lemma}
\newtheorem{proposition}[theorem]{Proposition}

\newtheorem{axiom}[theorem]{Axiom}

\newcommand{\NW}{\mathrm{NW}}
\newcommand{\MW}{\mathrm{MW}}

\newcommand{\UL}{\mathrm{UL}}
\newcommand{\GGW}{\mathrm{GGW}}

\newcommand{\rr}{\mathbf{r}}
\newcommand{\ee}{\mathbf{e}}
\newcommand{\F}{\mathbb{F}}
\newcommand{\veca}{\bm{a}}

\let\Pr\relax
\DeclareMathOperator{\Pr}{Pr}

\usepackage[toc,page,header]{appendix}
\usepackage{minitoc}

\title{Socially Fair Reinforcement Learning}

\author{
Debmalya Mandal\\
Max Planck Institute for Software Systems\\
 \texttt{dmandal@mpi-sws.org}\\ 
 \and Jiarui Gan\\
 University of Oxford\\
 \texttt{jiarui.gan@cs.ox.ac.uk}}
\begin{document}

\maketitle

\doparttoc
\faketableofcontents
\part{}

\begin{abstract}
We consider the problem of episodic reinforcement learning where there are multiple stakeholders with different reward functions. 
Our goal is to output a policy that is socially fair with respect to  different reward functions. Prior works have proposed different objectives that a fair policy must optimize including minimum welfare, and generalized Gini welfare. We first take an axiomatic view of the problem, and propose four axioms that any such fair objective must satisfy. We show that the Nash social welfare is the unique objective that uniquely satisfies all four objectives, whereas prior objectives fail to satisfy all four axioms.
We then consider the learning version of the problem where the underlying model i.e. Markov decision process is unknown. We consider the problem of minimizing regret with respect to the fair policies maximizing three different fair objectives -- minimum welfare, generalized Gini welfare, and Nash social welfare. Based on optimistic planning, we propose a generic learning algorithm and derive its regret bound with respect to the three different policies. For the objective of Nash social welfare, we also derive a lower bound in regret that grows exponentially with $n$, is the number of agents. Finally, we show that  for the objective of minimum welfare, one can improve regret by a factor of $O(H)$ for a weaker notion of regret. 
\end{abstract}

\section{Introduction}

In recent years, reinforcement learning has been immensely successful in various domains including  game playing~\cite{Silver+17}, and robotics~\cite{Andry+20}. These breakthroughs have opened up new avenues for applying reinforcement learning in real-world decision making systems like healthcare~\cite{GJKF+19}, and finance~\cite{LLXW+21}. However, such human-facing systems often include multiple stakeholders with different preferences, and the classical goal of maximization of rewards is no longer sufficient. A major design challenge is the selection of right objective that provides certain fairness guarantees across different agents. 
%Fairness is an important consideration in social planning, and it has been identified as one of the guiding principles for achieving AI for social good \cite{TCHM+20}. 
%To reach a fair solution in a multi-agent system requires identifying an objective that incorporates individual preferences and well balances individual gains. How to transform the multiple dimensions of this task to a concerted objective is the problem facing a social planner. The problem has been receiving growing attention across various application domains in AI and ML, thanks to the explosion of interactions and conflicts in AI systems.

% In this paper, we consider fairness in an episodic reinforcement learning problem where there are multiple stakeholders with different reward functions. 
% A single policy (or a joint policy) is to be decided and implemented, which will then generate different utilities for different stakeholders. 

Our first example comes from the use of reinforcement learning for vaccine distribution~\cite{MMMC21}. In this context, the vulnerable population is in immediate need, and their utilities should be prioritized. Similarly, multi-armed bandit algorithms have been deployed for improving access to maternal and child healthcare~\cite{MMTM+22}. In this domain, researchers have found that maximizing the minimum reward provides better fairness guarantees.

%Another example comes from negotiable reinforcement learning~\cite{Critch17, DCR18}, a sub-area of AI alignment research, where 

These examples suggest that in the context of multi-objective sequential decision-making, we should be careful about selecting the right measure. But what should be the right measure? Besides maximizing minimum welfare discussed above, researchers have proposed selecting a policy at the pareto frontier~\cite{DCR18}, or a policy maximizing generalized Gini social welfare~\cite{SWZ20}. In this paper, we focus on the problem of selecting a measure that ensures fairness in multi-objective sequential decision-making. For the static setting (i.e. single decision), this is a well-studied problem in economics~\cite{Moulin04}, and also in computational social choice~\cite{BCEL+16}. Motivated by the success of axiomatic framework in these fields, we aim to lay out a reasonable set of axioms for selecting a fairness measure in dynamic multi-objective sequential decision making systems.

Even though a set of axioms might uniquely characterize a fairness measure, the optimal policy according to this fairness measure is unknown. This is because the underlying MDP is unknown, and we must interact with the environment  to learn such an optimal policy. Moreover, given the use of different fairness measures in different contexts, can we come up with a generic reinforcement learning algorithm that learns different fair optimal policies?  This is the second question we study in this paper, and in particular, we consider the problem of minimizing regret with respect to different fair optimal policies. Note that, this question is challenging since most fairness measures are non-linear in value functions and dynamic programming ~\cite{Bertsekas12} based learning algorithms cannot be applied.

\subsection{Contributions}

At a high level, we propose a set of  axioms for selecting fairness measure in multi-objective sequential decision-making systems. Three natural  fairness measures are then  analyzed against the proposed axioms. 
When the underlying MDP is unknown, we propose a generic learning algorithm for minimizing regret with respect to different fairness measures.
In more detail, we make the following contributions.

\begin{itemize}
    \item 
    % In the context of episodic reinforcement learning with multiple reward functions, 
    We consider the episodic reinforcement learning with multiple agents, and propose a set of four axioms that any fairness measure should satisfy. We show that these axioms uniquely characterize Nash social welfare ($\NW$), whereas prior measures like minimum welfare ($\MW$) or generalized Gini welfare ($\GGW$) violates some of them.
    \item When the underlying model is unknown, we propose a generic learning algorithm for minimizing regret with respect to the optimal fair according to a fairness measure. For a setting with $n$ agents, $S$ states, $A$ actions, and horizon length $H$, we instantiate this algorithm for different measures, and obtain regret upper bounds of $\tilde{O}(n H^{n+1} S \sqrt{AT})$ for the objective of Nash social welfare, and $\tilde{O}(H^2 S \sqrt{AT})$ for both the minimum welfare, and generalized Gini welfare. 
    \item For the objective of Nash social welfare, we derive a lower bound in regret that scales exponentially with $n$, whereas for the other two measures lower bound immediately follows from lower bound in regret for single-agent setting. Finally, for the objective of minimum welfare, we show that one can improve regret by a factor of $O(H)$ for a weaker notion of regret.
\end{itemize}

\subsection{Related Work}

Our work is related to several lines of research.

\textbf{Multi-Objective Sequential Decision-Making}. Our work falls under the framework of multi-objective reinforcement learning~\cite{RVWD13, NYY09}. There are two standard approaches in multi-objective RL. Singly-policy algorithms convert the multi-objective problem into a single-objective RL by techniques like scalarization~\cite{MS01, Lin05, VKMN13, ARLN+19}. On the other hand, multiple-policy algorithms try to approximate the pareto frontier by learning a set of policies~\cite{NT05,BN08, MARW16, YSN19}. Our approach falls under the first category. However, to the best of our knowledge, prior work in multi-objective RL generally didn't consider learning optimal policies for non-linear  objectives with provable regret bounds. The only exception is \cite{Cheung19} who study RL with vectorial feedback in a \emph{non-episodic} setting. However, they consider a non-linear function of a particular form which cannot capture different fairness measures considered here.

\textbf{Nash Welfare in Bandits}. Beyond learning in Markov decision processes, our paper is closely related to \citet{HMS21} who consider the objective of Nash social welfare in multi-armed bandit and achieve $\tilde{O}(\sqrt{T})$ regret. Note that, our regret upper bound also becomes $\tilde{O}(\sqrt{T})$ if horizon length $H=1$, and number of states $S=1$. Additionally, there is no need to generalize the classical axioms of \cite{KN79} in a setting with a single state. \citet{BKMS22} also consider regret with respect to the Nash social welfare objective. However, they consider a setting where an agent arrives each round, and the goal is to maximize the Nash welfare (i.e. geometric mean of the rewards) over the $T$ rounds. \citet{BF21} also used Nash social welfare to axiomatize bargaining in a multi-armed grouped bandit setting. Beyond Nash welfare, several papers~\cite{BBLB20, PGNN21} have also studied other notions of fairness in multi-armed bandits.  

\textbf{Fairness in Multi-Agent Systems}. Our work is also related to the framework of fairness in multi-agent sequential decision-making~\cite{ZS14, JL19}. In a multi-agent MDP model, \citet{ZS14} propose to solve for a regularized maximin policy, where the regularizer offers a trade-off between utilitarian and max-min fair solution. \citet{JL19}, on the other hand, selects a policy in the pareto frontier. Recently, several papers have proposed to use generalized Gini welfare ($\GGW$)~\cite{SWZ20, ZGSW21} as a metric for selecting policies in multi-agent MDP. %We propose to choose the optimal policy according to the Nash  welfare.

\textbf{Fairness in Reinforcement Learning}. 
Our work is related to fairness in online learning ~\cite{JKMR16, JKMN+16, LRDMP17} and reinforcement learning~\cite{JJKMR17}. In contrast to our setting, these papers mainly define fairness with respect to the arms and stipulate that arms of lower quality should not be selected over arms of higher quality. Subsequent papers have also considered group fairness in online learning~\cite{HLW+20, SLMD19, WBT21}.%, but these fairness notions are different than ours.% with metrics like demographic parity. 
These notions of fairness are different than ours and the reader is referred to the recent survey~\cite{GSTF+22} for more details. 

\textbf{Episodic Reinforcement Learning}. 
The classic paper of \citet{AJO08} introduced the UCRL algorithm and studied regret minimization in average reward MDP setting. Our main learning algorithm is based on the optimistic planning approach of UCRL algorithm. In the context of episodic reinforcement learning \citet{AOM17} proposed value iteration based method UCB-VI which obtains  minimax optimal regret bound. However, value iteration based methods cannot be directly applied in our setting as fairness measures (e.g. minimum value) are non-additive.

%Subsequent papers have considered other versions of episodic RL including changing reward function~\cite{JJLS+20, RM19}, and changing transition function~\cite{JZBJ18, JL20}.
 
\begin{figure*}[!h]
\centering
\small
\begin{tabular}{@{}ccccccc@{}}
\toprule
\multicolumn{1}{l}{} &
  \multicolumn{4}{c}{Axioms} &
  \multicolumn{2}{c}{Regret} \\ \cmidrule(l){2-7} 
\multicolumn{1}{l}{} &
  \multicolumn{1}{l}{PO} &
  \multicolumn{1}{l}{ANON} &
  \multicolumn{1}{l}{IIAN} &
  \multicolumn{1}{l}{CON} &
  Upper Bound &
  Lower Bound \\ \midrule
% \rowcolor[HTML]{EFEFEF} 
Nash Social Welfare & \textbf{Y} & \textbf{Y} & \textbf{Y} & \textbf{Y} & $\tilde{O}\left(n H^{n+1} S \sqrt{AT}\right)$ & $\Omega\left(n \left(\frac{H}{2}  \right)^n\sqrt{SAT}\right)$ \\
Minimum Welfare &
  \textbf{N} &
  \textbf{Y} &
  \textbf{N} &
  \textbf{Y} &
  $\tilde{O}\left(H^2 S \sqrt{AT}\right)$ &
  $\Omega\left( \sqrt{HSAT}\right)$ \\
% \rowcolor[HTML]{EFEFEF} 
Generalized Gini Welfare &
  \textbf{Y} &
  \textbf{Y} &
  \textbf{N} &
  \textbf{Y} &
  $\tilde{O}\left(H^2 S \sqrt{AT}\right)$ &
  $\Omega\left( \sqrt{HSAT}\right)$ \\ \bottomrule
\end{tabular}
\caption{Comparison between different fairness measures for reinforcement learning. Measure $\NW$ is the unique fairness measure (up to a monotone transformation) satisfying all four axioms. The regret for $\NW$ scales exponentially with $n$, whereas for the other measure it is independent of $n$.\label{tab:results-summary}}
\end{figure*}

\section{Preliminaries}

We consider the episodic reinforcement learning problem with multiple reward functions. Before describing the general setting, we begin with the single-agent episodie reinforcement learning. We are given an MDP $\calM = (\calS, \calA, P, r, \rho)$ with $S$ states and $A$ actions i.e. $\abs{\calS} = S$ and $\abs{\calA} = A$.% We assume that the state-space is partitioned as $\calS = \sqcup_h \calS_h$ so that only states in $\calS_h$ are visited at time-step $h$. \footnote{We only use this assumption for learning algorithms, and our characterization result holds without this assumption.}
%\jiarui{seems both $\calA$ and $A$ are used.}
The initial state $s_1$ is drawn from a distribution $\rho$. We will write $s_h$ (resp. $a_h$) to denote the state visited (resp. action taken) at time-step $h$. For $h=1,\ldots,H-1$ the next state $s_{h+1} \sim P(\cdot|s_h, a_h)$. The goal is to maximize the expected sum of rewards over the $H$ steps i.e. $\E\left[\sum_{h=1}^H r(s_h,a_h) \mid s_1 \sim \rho \right]$. 

In this paper, we consider a setting with $n$ different reward functions $\rr = \{r_i\}_{i \in [n]}$ corresponding to $n$ different agents. We will assume $r_i(s,a) \in [0,1]$ for all $i$ and state,action pair $s,a$. Given a policy $\pi$, we can define the value function corresponding to the $i$-th reward function as follows.
\begin{equation}
V^\pi(\rho;r_i) = \E_\pi \left[\sum_{h=1}^H r_i(s_h,a_h) \mid s_1 \sim \rho \right]
\end{equation} 
Often the starting state distribution $\rho$ will be clear from the context and we will simply write $V^\pi(r_i)$ instead of $V^\pi(\rho; r_i)$.
Note that we are considering an episodic reinforcement learning setup, so the policy $\pi$ need not be a stationary policy. We will  write $\pi = (\pi_1,\ldots, \pi_H)$ where $\pi_h$ is the (non-stationary) policy used at time-step $h$. We will write $\Pi$ to denote the set of such non-stationary policies.

Any policy $\pi$ affects different reward functions differently (e.g. through value functions), and we want to build a measure to evaluate how fair the policy $\pi$ is with respect to $n$ different reward functions. 
% These measures capture how the value functions are distributed across the $n$ agents. %One can enforce the policy $\pi$ to satisfy certain fairness criteria.
In particular, a fairness measure $W$ maps a policy $\pi$ and a set of $n$ reward functions to a real number i.e. $W : \Pi \times \R^{S \times A \times n} \rightarrow \R$. 
Therefore, $W\left(\pi; \{r_i\}_{i \in [n]}\right)$ provides a way to evaluate the fairness of a policy $\pi$ and we want to maximize the measure to find the optimal fair policy. The most basic measure is the utilitarian measure which is the sum of values across the $n$ agents. However, this measure violates basic  axioms and we will consider the following three measures of fairness.

%\paragraph{Utilitarian.} 
%We maximize the social welfare, i.e., the sum of values across the $n$ agents.
%\[
%\pi^\star_\UL \in \argmax_\pi \sum_{i=1}^n V^\pi\left(r_i\right).
%\]
%Note that by linearity $\sum_{i=1}^n V^\pi\left(r_i\right)$ is also equivalent to $V^\pi\left(\sum_{i=1}^n r_i\right)$.

\paragraph{Minimum Welfare.} 
We maximize the minimum value function across the $n$ agents.
\[
\textstyle \pi^\star_\MW \in \argmax_\pi \min_{i \in [n]} V^\pi\left(r_i \right).
\]

%\deb{Can we say Minimum Welfare instead of Max-min Fair?}
%\jiarui{Yeah, I think that looks more consistent.}

\paragraph{Generalized Gini Social Welfare (GGW).}
This notion of fairness generalizes max-min fairness and has been considered previously in the literature on multi-objective Markov decision processes~\cite{SWZ20,ZGSW21}. We are given a vector of weights $w \in \R^n$ so that $w_i \ge 0$ for each $i$, $\sum_i w_i = 1$, and $w_1 \ge w_2 \ge \ldots \ge w_n$.
 Given a policy $\pi$, let $i_1,\ldots,i_n$ be an ordering of the agents so that $V^\pi(r_{i_1}) \le V^\pi(r_{i_2}) \le \ldots \le V^\pi(r_{i_n})$.  We then maximize the following objective:
$$
 \textstyle \pi^\star_\GGW \in \argmax_\pi \sum_{j=1}^n w_i V^\pi(r_{i_j}).
$$
When $w_1 = 1$, the above objective coincides with minimum welfare.
\paragraph{Nash Social Welfare.} 
We maximize the product of the value functions of the $n$ agents, which is known as the Nash social welfare.
\[
\textstyle \pi^\star_\NW \in \argmax_\pi \prod_{i=1}^n V^\pi(r_i).
\]

\medskip

When the full MDP is known, a policy that maximizes each of the measures can be computed by linear programming. The details can be found in the appendix. 

Prior work on (multi-objective) fair reinforcement learning has mostly focused on measures minimum welfare and generalized Gini welfare. In this work, we want to emphasize the Nash social welfare as an alternative measure of fairness because of its attractive axiomatic properties. 

\subsection{Axioms}

% We present a set of axioms which we hope a utility measure to satisfy and will investigate which of them are satisfied by each utility measures defined above. 
We view the problem of solving a fair reinforcement learning problem as maximization of a fairness measure $W$. Therefore, we propose four axioms that any such fairness measure $W$ should satisfy. These axioms are inspired by classical axioms in economics, particularly social choice theory~\cite{Sen18}. For our setting, we build on the axiomatic framework developed by \citet{KN79}.
% JG: below moved forward.
% In our setting, this function $W$ maps a policy $\pi$ and a set of $n$ reward functions to a real number i.e. $W : \Pi \times \R^{S \times A \times n} \rightarrow \R$. 
% In particular, $W\left(\pi; \{r_i\}_{i \in [n]}\right)$ evaluates to a fairness metric and we wish to maximize the metric. 
% We now provide a set of axioms that the function $W$ must satisfy. 
To define the axioms we will need some additional notations.
\begin{itemize}
\item Given a policy $\pi$ we will write $q^\pi$ to denote its state-action occupancy measure, i.e., 
\begin{equation}\label{eq:defn-occupancy-measure}
q^\pi_h(s,a) = \Pr_{\pi}\left(s_h = s, a_h = a \right).
\end{equation}

\item Given an occupancy measure $q$ we will write $\pi^q$ to denote the corresponding policy, i.e., 
\[
\pi^q_h(s,a) = 
\begin{cases}
\frac{q_h(s,a) }{ \sum_b q_h(s,b)} & \text{if } \sum_b q_h(s,b) > 0\\
\frac{1}{A} & \text{otherwise}
\end{cases}
\]
It is known that the occupancy measure of the policy $\pi^q$ is $q$. 
\item Given two occupancy measures $q_1$ and $q_2$, we will write $\pi^{\alpha q_1 + (1-\alpha) q_2}$ to denote the policy corresponding to the convex combination of the occupancy measures $q_1$ and $q_2$. In particular, first we define the occupancy measure $q = \alpha q_1 + (1-\alpha) q_2$ and then take the policy $\pi^q$. %the convex combination of the two policies, i.e., 
%\[
%\alpha \pi + (1-\alpha) \tpi = \left(\alpha \pi_1 + (1-\alpha) \tpi_1, \ldots, \alpha \pi_H + (1-\alpha) \tpi_H \right).
%\]
\end{itemize}

\begin{axiom}[\bf Pareto Optimality]\label{axiom:PO}
If $V^\pi(r_i) \ge V^{\tpi}(r_i)$ for all $i \in [n]$ and $V^\pi(r_j) > V^{\tpi}(r_j)$ for some $j$, then it must be that
\[
W\left(\pi; \rr \right) > W\left(\tpi; \rr \right).
\]
\end{axiom}
Namely, if policy $\pi$ pareto-improves policy $\tpi$, then $W$ must assign higher value to a policy $\pi$. 
\begin{axiom}[\bf Independence of Irrelevant Alternatives with Neutrality]\label{axiom:IIA}
Suppose that for all $i \in [n]$,
\[
\frac{V^{\pi_1}(r_i)}{V^{\pi_2}(r_i)} = \frac{V^{\tpi_1}(\tilde{r}_i)}{V^{\tpi_2}(\tilde{r}_i)}.
\]
%Let $q_j = q^{\pi_j}$ and $\tilde{q}_j = q^{\tilde{\pi}_j}$ for $j \in \set{1,2}$. Suppose the following holds for any $\alpha, \beta \in [0,1]$ and for any $i$.
%$$
%V^{\pi^{\alpha q_1 + (1-\alpha) q_2} }(\rho; r_i) \ge V^{\pi^{\beta q_1 + (1-\beta) q_2} }(\rho; r_i) \iff V^{\pi^{\alpha \tilde{q}_1 + (1-\alpha) \tilde{q}_2} }(\rho; \tilde{r}_i) \ge V^{\pi^{\beta \tilde{q}_1 + (1-\beta) \tilde{q}_2} }(\rho; \tilde{r}_i).
%$$
Then $W\left(\pi_1; \rr \right) \ge W\left(\pi_2; \rr \right)$ if and only if $W\left(\tpi_1; \tilde{\rr} \right) \ge W\left(\tpi_2; \tilde{\rr} \right)$.
% Then we must have
% \[
% W\left(\pi_1; \{r_i\}_{i \in [n]} \right) \ge W\left(\pi_2; \{r_i\}_{i \in [n]} \right) \iff W\left(\tpi_1; \{\tilde{r}_i\}_{i \in [n]} \right) \ge W\left(\tpi_2; \{\tilde{r}_i\}_{i \in [n]} \right).
% \]
\end{axiom}
Namely, if $\pi_1$ to $\pi_2$ is the same as $\tilde{\pi}_1$ to $\tilde{\pi}_2$ to each agent in terms of value function, then $W(\pi_1)$ to $W(\pi_2)$ is the same as $W(\tilde{\pi}_1)$ to $W(\tilde{\pi}_2)$. 

\begin{axiom}[\bf Anonymity]\label{axiom:ANON}
For any permutation $\sigma:[n] \rightarrow [n]$ of the agents and  policy $\pi$, it must be that
\[
W\left(\pi; \rr \right) = W\left(\pi; \{r_{\sigma(i)}\}_{i \in [n]} \right).
\]
\end{axiom}
Namely, $W$ is independent of the indices of the agents.

\begin{axiom}[\bf Continuity]\label{axiom:CON}
Suppose $W(\pi_1; \rr) \ge W(\pi_2; \rr) \ge W(\pi_3; \rr)$ and let $q_j = q^{\pi_j}$ for $j \in \set{1,2,3}$. Then there exists $\alpha \in [0,1]$ such that
\[
W(\pi_2; \rr) = W\left(\pi^{\alpha q_1 + (1-\alpha) q_3}; \rr \right).
\]
\end{axiom}

Namely, there is an intermediate policy between $\pi_1$ and $\pi_3$ that attains the same value as  $\pi_2$ under $W$. Note that, we don't take direct convex combination of policies, rather we take convex combination in the occupancy measure. This is because the value function in an MDP is a non-linear function of policy, and there might not be a policy of the form $\alpha \pi_1 + (1-\alpha) \pi_3$ achieving the same value as $\pi_2$.

%\deb{TODO:  comparison with alternative formulation (wrt value functions).}

% \section{Properties of Utility Measures}

\section{Axiomatic Analysis}

We investigate which of the above axioms are satisfied by different fairness measures.
First, we show that even though the axioms appear basic, not all of them are satisfied by the fairness measures considered previously in the literature. In fact, we can show that minimum welfare violates
PO and IIAN, whereas the measure generalized Gini welfare ($\GGW$) violates IIAN under any choice of weight vector $w$. 
On the other hand, we show that $\NW$ satisfies all four axioms. The main difference with $\GGW$, and $\MW$ is the satisfaction of axiom IIA. The proofs of these results are rather tedious, and they are provided in the appendix. Next state our first main result and show that up to a transformation by a monotonically increasing function, $\NW$ is the unique fairness measure that satisfies all four axioms.

%Here we provide the proof for the measure minimum welfare ($\MW$), and the proof for generalized Gini welfare ($\GGW$) is provided in the appendix.

%\subsection{Minimum and Generalized Gini Welfare}

% \textbf{Violation of axiom~\ref{axiom:PO}}: 

% \paragraph{Violation of Axiom~\ref{axiom:IIA}}: 

%Apart from the above violations, the remaining axioms are satisfied by the fairness measures minimum welfare, and the generalized Gini welfare. The proofs are straightforward and are provided in the appendix.
%\medskip

%\subsection{Nash Social Welfare} 

%We now turn to the fairness measure Nash social welfare. In the appendix, we show that $\NW$ satisfies all four axioms. The main difference with $\GGW$, and $\MW$ is the satisfaction of axiom IIA. Next we state our first main result and show that up to a transformation by a monotonically increasing function, $\NW$ is the unique fairness measure that satisfies all four axioms.

\begin{theorem}
\label{thm:NW-unique}
% Axioms~\ref{axiom:PO}, \ref{axiom:IIA}, \ref{axiom:ANON}, and \ref{axiom:CON} uniquely characterizes $\NW$. 
Suppose there are at least four actions available at each state, and  $W$ is a  fairness measure satisfying Axioms~\ref{axiom:PO}, \ref{axiom:IIA}, \ref{axiom:ANON}, and \ref{axiom:CON}.
Then it holds that $W(\pi; \rr)= \phi \left(\NW(\pi; \rr) \right)$ for any policy $\pi$, and reward function $\rr$, where $\phi: \R \rightarrow \R$ is some monotonically increasing function.
\end{theorem}

The proof of Theorem~\ref{thm:NW-unique} is provided in the appendix, and adapts  main ideas from \citet{KN79} to the setting of episodic reinforcement learning. 
\begin{enumerate}
    \item We first define an order $\succcurlyeq$ on $\R^n_+$ as $x \succcurlyeq y$ if and only if $W(\pi^a; \rr) \ge W(\pi^b; \rr)$. Here $\pi^a$ (resp. $\pi^b$) always plays action $a$ (resp. $b$) and $\rr$ is a reward function constructed using $x$ and $y$.
    \item Based on the function $W$, we then define a function $F: \R^n_+ \rightarrow \R$ that respects the order $\succcurlyeq$. The construction of this function crucially depends on axiom~\ref{axiom:CON}.
    \item We then show the function $F$ satisfy some requirements so that we can invoke a result due to \cite{osborne1976irrelevant} and obtain that there exists a monotone increasing function $G$ such that $F(x) = G\left(\prod_{i=1}^n x_i\right)$. This enables us to show that $W$ is a monotone increasing transformation of $\NW$.
\end{enumerate}

In addition to the above properties, the Nash social welfare also provides a utility guarantee in comparison to the optimal max-min fair policy: $\pi^\star_\NW$ provides any agent at least $1/n$ fraction of her value under $\pi^\star_\MW$.

\begin{theorem}\label{thm:nw-to-max-min}
Suppose that the maximum $\MW$ social welfare is not zero.\footnote{When the maximum $\MW$ social welfare is zero, for each policy $\pi$ there exists an agent $i$ such that $V^\pi(r_i) = 0$, i.e., $\NW(\pi) = 0$. This implies that $\pi^\star_\NW$ is also a max-min fair policy.}%The case where the MM social welfare is zero is a degenerate case where one agent gets reward zero for every state-action pair.} 
Then for any agent $i \in [n]$, it holds that
\[
V^{\pi^\star_\NW}(r_i) \ge \frac{1}{n} V^{\pi^\star_\MW}(r_i).
\]
Moreover, this bound is tight, i.e., for any $\delta > 0$ there exists an instance s.t.
$V^{\pi^\star_\NW}(r_i) \le \frac{1}{n} V^{\pi^\star_\MW}(r_i) + \delta$ for some $i$.
\end{theorem}

% The proof of Theorem~\ref{thm:nw-to-max-min} is provided in the appendix. 
The main idea to prove Theorem~\ref{thm:nw-to-max-min} is to show that the optimal policy $\pi^\star_\NW$ can be computed by optimizing a log-concave function over a convex set, and then use first order optimality conditions at the optimal solution.\footnote{A similar relationship is known in the fair division literature~\cite{CKMP+19}, who showed a tight bound of $1/\sqrt{n}$, but their setting and proof techniques are very different than ours.} 
The details can be found in the appendix.

\section{Learning}
We now consider the setting where the probability transition function $P$ is unknown and the learner needs to learn $P$ in order to compute a fair policy. 
The learner interacts with the environment over $T$ episodes. We will write $\pi_t$ to denote the policy (possibly non-stationary) adopted by the learner during episode $t$. We will also use $s_{t,h}$ to denote the state visited at time-step $h$ of episode $t$, and $a_{t,h}$ to denote the action taken at time-step $h$ of episode $t$. 
As is standard in the literature on online reinforcement learning, we will assume that the reward functions are known, but we show in the appendix how our algorithm and analysis can be easily generalized to handle unknown reward functions. %\deb{TODO: Add a section in the appendix.} 
We next define the regret of a learner that interacts with the world over $T$ episodes.

\newcommand{\LINECOMMENT}[1]{\STATE {\color{blue}\ttfamily\small \(\triangleright\) #1}}

\begin{algorithm}[!h]
\caption{Upper Confidence Reinforcement Learning for Fair Objective $\F$ (UCRL-$\F$)\label{algo:ucrl-fair}}
\begin{algorithmic}
\FOR{$t=1,2,\ldots$}{
%\tcc{Epoch $j$ runs from $2^j$ to $2^{j+1} - 1$ episodes.}
\LINECOMMENT{Compute optimistic MDP $\tilde{P}_t$ and the corresponding policy $\tilde{\pi}_t$}
\vspace{-3mm}
 \STATE 
 \begin{align}
\label{eq:optimistic planning}
\begin{split}
%\begin{array}{l}
    &\tilde{P}_t \leftarrow \argmax_{P \in C_t(\widehat{P})} \max_\pi \F(\pi; P) \\
    &\tilde{\pi}_t \in \argmax_\pi \F(\pi; \tilde{P}_t)
%\end{array}
\end{split}
\end{align}

 \FOR{$h=1,\ldots,H$}{
 \STATE Observe state $s_{t,h}$.
 \STATE Play $a_{t,h} \sim \tilde{\pi}_{t,h}(\cdot|s_{t,h})$.\\
 }
 \ENDFOR
\LINECOMMENT{Update estimates of $P$}
\STATE $N_t(s,a,s') = \sum_{t' \le t} \sum_{h=1}^{H-1} \calE_{t,h}(s,a)$
\STATE with $\calE_{t,h}(s,a) =\one\set{s_{t,h} = s, a_{t,h} = a, s_{t,h+1}=s'}$
\STATE $N_t(s,a) = \sum_{t' \le t} \sum_{h=1}^H \one\set{s_{t,h} = s, a_{t,h} = a}$

\STATE $\varepsilon_t(s,a) = \sqrt{\frac{4 S \log(SAt/\delta)}{\max\set{1,N_t(s,a)}} }$
\FOR{$s,a,s'$}
{
\STATE $\widehat{P}(s,a,s') = \frac{N_t(s,a,s')}{N_t(s,a)}$.
}
\ENDFOR
\STATE $C_t(\widehat{P}) = \set{P: \forall s\ \norm{P(s,a,\cdot) - \widehat{P}(s,a,\cdot)}_1 \le \varepsilon_t(s,a)}$
}
\ENDFOR
\end{algorithmic}
\end{algorithm}

\paragraph{Regret.} We define regret with respect to a generic fair objective $\F \in \set{\NW, \GGW, \MW}$. For a policy $\pi$ we will write $\F(\pi)$ to denote its value according to the objective $\F$ e.g.  $\NW(\pi) = \prod_{i=1}^n V^\pi(\rho; r_i)$. Let $\pi^\star_\F$ be the policy that maximizes the objective $\F$ i.e. 
$
    \pi^\star_\F \in \argmax_{\pi} \F(\pi).
$
Then we measure the regret of a learning algorithm $(\pi_1,\ldots,\pi_T)$ with respect to this optimal fair policy.
\begin{equation}\label{eq:regret-defn-F}
    \reg_\F(T) = \sum_{t=1}^T \left(\F(\pi^\star_\F) - \F(\pi_t) \right)
\end{equation}

\paragraph{Algorithm.} Algorithm~\ref{algo:ucrl-fair} presents our learning algorithm for a generic objective $\F$. The algorithm is based on the principle of optimistic planning, introduced by \citet{AJO08}. The learner interacts with the environment over $T$ episodes each of length $H$. At the start of  episode $t$, the learner first computes the the optimistic model $\tilde{P}_t$ and the optimistic policy $\tilde{\pi}_t$ (\cref{eq:optimistic planning}). Then the learner plays policy $\tilde{\pi}_t$ for $H$ steps, and at the end of the episode, updates the empirical probability transition function $\widehat{P}$.

\paragraph{Optimistic Planning.} 
The most important step of algorithm~\ref{algo:ucrl-fair} is the computation of an optimistic model and a policy (\cref{eq:optimistic planning}). Here $\tilde{P}_t$ is the probability transition function that is plausible at time $t$ (i.e., lies within the confidence ellipsoid $C_t(\widehat{P})$ centered around $\widehat{P}_t$) and has the highest possible objective value according to the function $\F$. For an arbitrary objective $\F$, it is not even clear that this step can be performed efficiently. However, we show that for objectives $\NW$ and $\MW$ the optimistic planning step can be performed efficiently i.e. time polynomial in $S,H$, and $A$. On the other hand, for the objective $\GGW$ we can solve optimistic planning in time $n^{O(n)} \cdot \textrm{poly}(S,A,H)$. This implies that as long as the number of agents $n$ is constant, we can also solve the optimistic planning of $\GGW$ efficiently.

The main idea is to consider an extended MDP where the state space is the same but the action at any state $s$ is all tuples of the form $(a, \tilde{P}_{s,a})$ where $\widetilde{P}_{s,a}$ is a probability distribution over next states so that $\norm{\widetilde{P}_{s,a} - \widehat{P}(s,a,\cdot)}_1 \le \epsilon_t(s,a)$. We then show how to compute the optimal policy according to the different objectives in the extended MDP. In particular, we set up an optimization problem with variables $q_h(s; a, \widetilde{P}_{s,a})$ which denotes the probability a policy takes action $(a, \widetilde{P}_{s,a})$ at state $s$ at time-step $h$. The number of such variables is infinite. However, for the objectives $\MW$ and $\NW$, the number of constraints is bounded by $n + HS$. So we consider the dual problem and show that the dual can be solved efficiently using the ellipsoid method. However, for the objective of $\GGW$, the primal problem has $n!+ HS$ constraints. So we end up with a dual problem with $n!+HS$ variables, and solving this dual problem requires time $n^{O(n)} \cdot \textrm{poly}(S,A,H)$. We believe that it might be possible to get a PTAS for the optimistic planning with objective $\GGW$ but we leave this question for the future. The details of solving the optimistic planning steps are provided in the appendix. We now turn to proving the regret guarantees for algorithm~\eqref{algo:ucrl-fair}.

\begin{theorem}\label{thm:regret-nsw}
For the objective of Nash social welfare, Algorithm~\ref{algo:ucrl-fair} has regret 
$$
\reg_\NW(T) = 
\tilde{O}\left(n H^{n+1} S\sqrt{AT}\right).
$$
\end{theorem}

\begin{proof}[Proof sketch]
The proof of this theorem is provided in the appendix. Here we discuss the main ideas behind the proof.
\begin{enumerate}
    \item We can apply Chernoff-Hoeffding inequality and the union bound to show that, with high probability,
    the true probability transition function $P^\star$ is contained within the confidence ellipsoid $C_t(\widehat{P})$ for all time steps $t \in \set{1,2,\ldots,T}$ i.e.
    $
    P\left( \exists t\ P^\star \notin C_t(\widehat{P})\right) \le \delta.
    $
    \item Conditioned on the event above, we show that the regret can be upper bounded as $\reg_\NW(T) \le \sum_{t=1}^T \NW(\tilde{\pi}_t; \tilde{P}_t) - \NW(\tilde{\pi}_t; P^\star)$. This step uses the fact that $\tilde{P}_t$ is the optimistic model and $\tilde{\pi}_t$ is the optimistic policy at time $t$.
    
    \item 
    % We then use \Cref{lem:nsw-to-value} to 
    We can then bound the difference $\NW(\tilde{\pi}_t; \tilde{P}_t) - \NW(\tilde{\pi}_t; P^\star)$ in Nash welfare in terms of sum of difference in value functions. This gives us the following upper bound on regret (see \cref{lem:nsw-to-value} in the appendix).
    $$
    H^{n-1} \sum_{i=1}^n  \underbrace{\sum_{t = 1}^T  \abs{V^{\tilde{\pi}_t}\left(r_i; \tilde{P}_t\right) - V^{\tilde{\pi}_t}\left(r_i; P^\star\right) }}_{:= \reg_i}
    $$
    % Lemma~\ref{lem:nsw-to-value} 
    This result can be thought of as a linearization lemma, and it generalizes Lemma~2 by \citet{HMS21} to the general setting of finite-horizon reinforcement learning.
    \item Finally, we prove a bound of $\tilde{O}\left( H^2 S\sqrt{AT}\right)$ on the term $\reg_i$ to complete the proof.
    \qedhere
\end{enumerate}
\end{proof}

% \begin{restatable}{lemma}{lmmnswtovalue}
% \label{lem:nsw-to-value}
% For any policy $\pi$, and probability transition functions $P_1$ and $P_2$ we have
% $\NW(\pi; P_1) - \NW(\pi; P_2) \le H^{n-1} \sum_{i=1}^n \abs{V^\pi(r_i; P_1) - V^\pi(r_i; P_2)}$.
% \end{restatable}

We now consider the regret in learning the max-min fair policy. We instantiate algorithm~\ref{algo:ucrl-fair} with objective $\F = \MW$ and obtain a regret bound of $\tilde{O}(H^2 S \sqrt{AT})$. Note that, unlike the objective of Nash social welfare, the regret in this case doesn't grow with the number of agents $n$. % Let $\pi^\star_\MW$ be the optimal max-min fair policy i.e.  %
%$$
%\pi^\star_\MW \in \argmax_\pi \min_i V^\pi(\rho; r_i)
%$$
%Then we measure regret of a learning algorithm $(\pi_1,\ldots,\pi_T)$ with respect to the optimal max-min fair policy.
%\begin{equation}
%\label{eq:defn-reget-maximin}
%\reg_\MW(T) = \sum_{t=1}^T \min_i V^{\pi^\star_\MW}(r_i) - \min_i V^{\pi_t}(r_i)
%\end{equation}
%We adopt an algorithm very similar to the UCRL-NW algorithm except that we now compute optimistic model according to the maximin objective. In particular, at the end of episode $t$ UCRL-MM performs the following two steps.
%$$
%\tilde{P}_t \leftarrow \argmax_{P \in C_t(\hat{P})} \max_\pi \min_i V^\pi(r_i; P)\quad \textrm{and}\quad \tilde{\pi}_t \in \argmax_\pi \min_i V^\pi(r_i; \tilde{P}_t)
%$$
%The next theorem shows that UCRL-MM has regret sub-linear in $T$.
\begin{theorem}\label{thm:regret-mm}
For the objective of minimum welfare, Algorithm~\ref{algo:ucrl-fair} has regret
$$
\reg_\MW(T) = \tilde{O}(H^2 S \sqrt{AT}).
$$
\end{theorem}
\begin{proof}
Proceeding similarly as the proof of Theorem \ref{thm:regret-nsw} we can show that the regret can be upper bounded by the regret with respect to the optimistic policies.
\begin{align*}
\reg_\MW(T) &\le \sum_{t=1}^T \min_i V^{\tpi_t}(r_i; \tilde{P}_t) - \min_i V^{\tpi_t}(r_i; P^\star)
%&\le \sum_{t=1}^T \min_i \left\{\abs{V^{\tpi_t}(r_i; \tilde{P}_t) - V^{\tpi_t}(r_i; P^\star)}\right. \\
%&+ \left. V^{\tpi_t}(r_i; P^\star)\right\} - \min_i V^{\tpi_t}(r_i; P^\star)\\
%&\le \sum_{t=1}^T \min_i \abs{V^{\tpi_t}(r_i; \tilde{P}_t) - V^{\tpi_t}(r_i; P^\star)}
\end{align*}
We can bound the term $\min_i V^{\tpi_t}(r_i; P^\star)$ as follows.
\begin{align*}
    &\min_i V^{\tpi_t}(r_i; P^\star) \\&= \min_i \left\{ \left(V^{\tpi_t}(r_i; P^\star) - V^{\tpi_t}(r_i; \tilde{P}_t) \right) + V^{\tpi_t}(r_i; \tilde{P}_t) \right\}\\
    &\ge -\max_i \abs{V^{\tpi_t}(r_i; P^\star) - V^{\tpi_t}(r_i; \tilde{P}_t)} + \min_i V^{\tpi_t}(r_i; \tilde{P}_t) 
\end{align*}
% The last inequality follows because the rewards are bounded between $0$ and $1$ which also guarantees that $V^{\tpi_t}(r_i;\tilde{P}_t) \ge 0$ for any $i$. 
Substituting this bound gives us the following upper bound on $\reg_\MW(T)$.
\begin{align*}
    &\reg_\MW(T) \le \sum_{t=1}^T \max_i \abs{V^{\tpi_t}(r_i; P^\star) - V^{\tpi_t}(r_i; \tilde{P}_t)}\\
    &\le \underbrace{\sum_{t=1}^T H \sum_{s',b} \tilde{q}_h(s',b) \epsilon(s',b) }_{:= \widetilde{\reg}} \quad \textrm{[By Lemma~\ref{lem:diff-in-value}]}
\end{align*}
where $\epsilon_t(s',a) = \norm{\tilde{P}_t(s',a,\cdot) - P^\star(s',a,\cdot)}_1$ and $\tilde{q}$ denotes the state, action occupancy measure under policy $\tpi$.  The proof of Theorem~\ref{thm:regret-nsw} first upper bounds $\reg_i$ by $\widetilde{\reg}$ for any $i$ and then bounds $\widetilde{\reg}$ by $\tilde{O}(H^2 S \sqrt{AT})$.
%Since minimum is bounded by average we can bound regret as
%\begin{align*}
%\reg_\MW(T) \le \frac{1}{n} \sum_{i=1}^n \underbrace{\sum_{t=1}^T \abs{V^{\tpi_t}(r_i; \tilde{P}_t) - V^{\tpi_t}(r_i; P^\star)}}_{:=\reg_i}
%\end{align*}
%Since Theorem \ref{thm:regret-nsw} provides an upper bound of $\tilde{O}(H^2 S \sqrt{AT})$ on each $\reg_i$, the proof is complete.
\end{proof}

The next theorem upper bounds the regret for the objective of Generalized Gini Social welfare.
\begin{theorem}\label{thm:regret-ggw}
For the objective of generalized Gini social welfare, Algorithm~\ref{algo:ucrl-fair} has regret
$$
 \reg_\GGW(T) = \tilde{O}\left(H^2 S \sqrt{AT}\right).
$$
\end{theorem}
The proof of this theorem is similar to theorem~\ref{thm:regret-mm} and is provided in the appendix.

 \subsection{Lower Bound on Regret}
 %\deb{The description of the lower bound instance can be shortened, since it is similar to prior work.}
 If all the $n$ agents have identical reward functions then regret in learning minimum welfare ($\reg_\MW(T)$) just corresponds to regret in learning a single reward function. Therefore, $\reg_\MW(T)$ must be at least the regret in learning for the single-agent setting which is $\Omega(\sqrt{HSAT})$~\cite{AOM17}. Now, for the generalized Gini welfare recall that the weights are normalized i.e. $\sum_k w_k = 1$. So, the regret in learning ($\reg_\GGW(T)$) again coincides with the regret in learning for the single-agent setting and we have $\reg_\GGW(T) \ge \Omega(\sqrt{HSAT})$.
 
 We now consider the regret in learning for the Nash welfare objective. 
 Note that the upper bound in regret scales exponentially with $n$, the number of agents. Intuitively this is expected as in an episodic reinforcement learning setting, the Nash social welfare of $n$ agents scales as $H^n$. We now provide a lower bound in regret $\reg_\NW(T)$ that also scales exponentially with $n$. 
 
  \begin{theorem} \label{thm:regret-nsw-lower-bound}
  Suppose $T \ge \Omega(\log H)$ and $H \ge \Omega(\log S)$. Then
 for any policy $\pi$, there exists a MDP with $S$ states, and $A$ actions so that $\E[\reg_\NW(\pi;T)] \ge C n \left( \frac{H}{2}\right)^n \sqrt{SAT}$ for some universal constant $C > 0$.
 \end{theorem}
 
Here we briefly highlight the main ideas of the proof. 
 We use the same instance used in the lower bound construction of average reward MDP~\cite{LS20, AJO08}, where there is a collection of $O(SA)$ MDPs. Each MDP has $S-2$ states arranged in the form of an $A$-ary tree, and two other states -- good state $s_g$ (reward $1$) and bad state $s_b$ (reward $0$). Any action from the leaves of this tree transition uniformly at random to the remaining two states are good ($s_g$) and bad ($s_b$). However, on model $M_{\ell,a}$ we define
 \begin{align*}
\textstyle P(s_g|\ell,a) = {1}/{2} + \Delta,\ P(s_b|\ell,a)= {1}/{2} - \Delta.
 \end{align*}
 
 Now observe that on Model $M_{\ell,a}$ the optimal policy can navigate to leaf $\ell$ and take action $a$. Moreover, the transition probabilities are designed in such a way that once either the good or the bad state is reached, it takes at least $\Theta(H)$ steps to leave that state. This means the optimal policy achieves Nash welfare of at least $(1/2 +\Delta)^n H^n$ per episode. On the other hand, if the learning policy does not take action $a$ at leaf node $\ell$, it achieves a Nash welfare of at most $(1/2)^n H^n$ that episode. Therefore, the regret in these episodes is at least $\left((1/2 +\Delta)^n  - (1/2)^n \right)H^n \ge n (H/2)^n$. Then the rest of the proof shows that one can choose $\Delta$ small enough so that there exists a model under which the number of such misses by the learning algorithm is at least $O(T)$.

\subsection{Improved  Bound for Minimum Welfare}
We now consider a weaker notion of regret for the fairness measure $\MW$, minimum social welfare.
\begin{equation}\label{eq:new-defn-regret}
    \reg_\MW^W(T) = \sum_{t=1}^T \MW(\pi^\star_\MW; \rr) - \min_i \sum_{t=1}^T V^{\pi_t}(r_i)
\end{equation}
The difference with the regret definition~\eqref{eq:regret-defn-F} is that for the second quantity the order of minimum and summation is interchanged. Since  value functions are non-negative, sum of minimum welfare across $T$ episodes is upper bounded by the minimum of the total welfare across $T$ episodes. This implies that $\reg^W_\MW(T) \le \reg_\MW(T)$, and we show that  it is  possible to improve the bound for $\reg^W_\MW(T)$.

Before deriving a regret bound with respect to the new definition, we first discuss when  the new  regret~\cref{eq:new-defn-regret} is more appropriate. Consider a setting where all the $n$ agents interact with the centralized learner for $T$ episodes and cares about her total reward over the $T$ episodes. %Since each of the $n$ agents interact with the learner for $T$ episodes, 
In such a setting, the total utility received by agent $i$ is $\sum_{t=1}^T V^{\pi_t}(r_i)$. Therefore, the minimum total welfare  is $\min_i \sum_{t=1}^T V^{\pi_t}(r_i)$. On the other hand, the optimal minimum welfare achievable over $T$ episodes is 
$$
\textstyle \max_{\pi \in \Pi} \min_i \sum_{t=1}^T V^\pi(r_i) = T \max_{\pi \in \Pi} \min_i V^\pi(r_i).
$$
The difference between these two quantities is precisely the regret in \cref{eq:new-defn-regret}.

Algorithm~\ref{algo:lagrange-maximin} describes our new learning algorithm and is inspired by the dual formulation of the problem of solving max-min fair policy. 
\begin{align}\label{eq:max-min-LP-main-text}
\begin{split}
&\max_{v \in \R,\ q \in \calQ(\rho, P)}\ v\\
\textrm{s.t.} &\sum_{h=1}^H \sum_{s,a} q_h(s,a)  r_i(s,a) \ge v \ \forall i \in [n]\\
 %& q \in \calQ(\rho, P) %\sum_{a} q_1(s,a) = \rho(s) \ \forall s\\
%&\sum_a q_{h+1}(s,a) = \sum_{s'} \sum_a q_h(s',a) P(s',a,s)\\
%&\qquad \qquad \forall s, \ \forall h \in [H-1]\\
%&q \ge 0
\end{split}
\end{align}
Here $q$ is the state-action occupancy measure (as defined in \cref{eq:defn-occupancy-measure}) and the constraint $q \in \calQ(\rho,P)$ ensures that $q$ is valid with respect to $P$. The  Lagrangian of the optimization problem~\eqref{eq:max-min-LP-main-text} is given as
\begin{align*}
\calL(q,v;\lambda) &= v + \sum_{i=1}^n \lambda_i \left( \sum_{h=1}^H \sum_{s,a} q_h(s,a) r_i(s,a) - v\right)
\end{align*}
Suppose we know $v^\star$, the optimal solution to \eqref{eq:max-min-LP-main-text}, i.e., the maxmin value. Then substituting $v=v^\star$ we get the following expression for the Lagrangian.
\begin{align*}
\calL(q;\lambda) &= v^\star + \sum_{i=1}^n \lambda_i \left( \sum_{h=1}^H \sum_{s,a} q_h(s,a) r_i(s,a) - v^\star\right)
\end{align*}
Here we assume that $v^\star$ is known and it is an input to Algorithm~\ref{algo:lagrange-maximin}. In our full proof we show that it suffices to set $v^\star$ to an upper bound of the maxmin value in Algorithm~\ref{algo:lagrange-maximin} (e.g., we can use an upper bound of the rewards).
%\jiarui{Maybe we can directly say $v^\star$ is an upper bound of maximin in Algorithm~\ref{algo:lagrange-maximin}.}

The Lagrangian can be interpreted as a two-player zero-sum game where the learner (max-player) plays $q$ and the adversary (min-player) plays $\lambda$. %The reward for the learner is $\calL(q,\lambda)$, and the reward for the adversary is $-\calL(q,\lambda)$. The goal of the learner is to maximize its reward and the goal of the adversary is to minimize its cost. 
Algorithm~\ref{algo:lagrange-maximin} tries to find an equilibrium of this game by using a reinforcement learning algorithm for the $q$-player and a bandit algorithm for the $\lambda$-player. In order to see which RL algorithm is suitable let us first rewrite the term $\calL(q,\lambda)$ as 
\begin{align*}
     \sum_{h=1}^H \sum_{s,a} q_h(s,a) \left( \sum_i \lambda_i r_i(s,a) + \frac{v^\star}{H}(1 - \sum_i \lambda_i) \right)
\end{align*}
Therefore, we can define a new  environment with rewards $\tilde{r}(s,a) = \sum_i \lambda_i r_i(s,a) + \frac{v^\star}{H} \left(1 - \sum_i \lambda_i \right)$ and in the new environment the RL agent receives the expected reward $\calL(q,\lambda)$. We will fix the choice of $\lambda$ at the start of an episode (say $\lambda^t$ at the start of episode $t$), and we also need to fix the policy at the start of the episode. For this reason, we will be using the UOB-REPS algorithm~\cite{JJLS+20} which works with unknown transition, and adversarial rewards, and has regret  $O(HS\sqrt{AT})$.

The $\lambda$-player, on the other hand, chooses a point in the convex set $\calC = \set{\lambda \in \R^n_+ : \sum_i \lambda_i \le B }$ at the start of each episode and receives an expected cost of $\calL(q,\lambda)$. As the cost function is linear in $\lambda$ we can use any  algorithm for linear bandit with adversarial reward.
We will use the OSMD algorithm proposed by ~\citet{BCK12} which has regret ${O}(\sqrt{nT\log A})$. %Algorithm~\ref{algo:lagrange-maximin} provides full details of the algorithm.

\begin{algorithm}[!t]
\caption{Lagrange-Maximin~\label{algo:lagrange-maximin}}
\begin{algorithmic}
\STATE \textbf{Input:} Maxmin value $v^\star$, number of episodes $T$, length of each episode $H$.
\STATE Instantiate \texttt{UOB-REPS} with action set $\calA$ and state space $\calS$.
\STATE Instantiate OSMD with action set $\calC = \set{\lambda \in \R^n_+: \sum_i \lambda_i \le B}$.
\FOR{$t=1,\ldots,T$}{
\STATE $\pi^t \leftarrow$ policy chosen according to UOB-REPS.
\STATE $\lambda^t \leftarrow$ action chosen according to OSMD.
\FOR{$h=1,\ldots,H$}{
\STATE Observe $x_{t,h}$.
\STATE Action $a_{t,h} = \pi^t(x_{t,h})$.
\STATE Reward feedback $\tilde{r}_{t,h} = \sum_i \lambda^t_i r_i(x_{t,h}, a_{t,h}) + \frac{v^\star}{H} \left( 1 - \sum_i \lambda^t_i\right)$.
}
\ENDFOR
\STATE Loss feedback to OSMD: $v^\star + \sum_{i=1}^n \lambda^t_i \left( \sum_{h=1}^H r_i(x_{t,h}, a_{t,h}) - v^\star\right)$.
}
\ENDFOR
\end{algorithmic}
\end{algorithm}

\begin{theorem}\label{thm:bound-new-regret}
For $n \le \frac{S^2 A}{\log A}$, algorithm~\ref{algo:lagrange-maximin} has regret
$$
\reg_\MW^W(T) = \tilde{O} \left( HS \sqrt{AT}\right)
$$
\end{theorem}
%The proof of theorem can be found in the appendix. 
The main idea is to show that the average occupancy measure $\bar{q} = \frac{1}{T} \sum_t q^t$ and $\bar{\lambda} = \frac{1}{T}\sum_t \lambda^t$ is an $\epsilon$-approximate fixed point of the game $\calL(h,q)$ for $\epsilon = \tilde{O}\left(B(HS\sqrt{A} + \sqrt{n \log A}) / \sqrt{T} \right)$. This lets us bound the constraint violation of $\bar{q}$ with respect to the LP~\eqref{eq:max-min-LP-main-text} and also bound regret $\reg^W_\MW(T)$. The proof also shows that any choice of $v^\star$ larger than $v^\star_\MW = \max_\pi \min_i V^\pi(r_i)$ works, so one can call algorithm~\ref{algo:lagrange-maximin} with $v^\star = H$. Additionally, the parameter $B$ needs to be at least the maximum possible value of any reward function.
% \newpage 
%\printbibliography

\section{Discussion and Future Work}

\paragraph{Multi-Agent Setting:} Our setting can be extended to multi-agent setting where each agent plays her own action. In particular, suppose the state space is $\calS$ but the action space is $\calA^n = \set{\veca = (a_1,\ldots,a_n) : a_i \in \calA \ \forall i}$. Then we can show that the fair policies according to the objective ($\MW, \GGW$, and $\NW$) can be computed in time polynomial in $S,H$, and $A^n$. Indeed our current approach is to set up a convex optimization problem where the variables are $q_h(s,a)$, the probability of visiting state $s$ at time-step $h$ and taking action $a$. For the multi-agent setting the set of variables become $q_h(s,\veca)$ for each $h,s, \veca \in \calA^n$, and solving the corresponding optimization problem takes time exponential in $n$. Our learning algorithm~\eqref{algo:ucrl-fair} also generalizes for the multi-agent setting if one maintains an estimate of $P(s,\veca, s')$ for each $s,s'$ and action profile $\veca$. However, the regret now scales exponentially with $n$ (e.g. $\reg_\NW = O\left(n H^{n+1} S A^{n/2} \sqrt{T}\right)$) without any assumption on the reward functions of the agents.

\paragraph{Conclusion}
In this paper, we proposed a set of axioms for selecting fairness measure in multi-agent sequential decision making systems. We analyzed three different fairness measures ($\NW, \MW$, and $\GGW$) against the  axioms. 
When the underlying MDP is unknown, we proposed a generic learning algorithm for minimizing regret with respect to different fair optimal policies. 
There are many interesting directions for future work. First, 
it would be interesting to extend our framework to consider RL with general function approximation. Our axioms should generalize easily, but the computation of fair optimal policies might be challenging. Second, 
it would be great to bridge the gap between the lower and upper bounds in regret (table~\ref{tab:results-summary}). For single-agent RL, value iteration based methods achieve minimax regret bounds. However, fairness measures are non-linear, and 
we might need to devise different  learning algorithms. Finally, it would also be interesting to consider a setting where  agents have partial information about the underlying MDP~\cite{DCR18}.

%\bibliographystyle{apalike}
%\bibliography{refs}
\printbibliography

% \onecolumn

\clearpage
\appendix
\addcontentsline{toc}{section}{Appendix}
\part{Appendix}
\parttoc

\section{Computing Optimal Policies}

We first recall the linear programming formulation of reinforcement learning. Consider an MDP $\calM = (\calS, \calA, P, r, \rho)$ and our goal is to maximize value function with respect to the reward function $r$.
$$
\max_\pi V^\pi(\rho; r_i)
$$
The primal linear program for this problem is the following.
\begin{align*}
\min_{\{V_h\}_{h \in [H]}}& \sum_s \rho(s) V_1(s)\\
\textrm{s.t.} & V_h(s) \ge r(s,a) + \sum_{s'} P(s,a,s') V_{h+1}(s')\ \forall h < H
\end{align*}
We will use the dual formulation of the above linear program.
%Following~\cite{NP20}, we can write down a linear programmings formulation of the above objective as follows.  %Let $q_h(s,a) = P\left(s_h = s, a_h = a)$.
\begin{align}\label{eq:dual-LP-RL}
\begin{split}
\max_q &\sum_{h=1}^H \sum_{s,a} q_h(s,a)  r(s,a)\\
\textrm{s.t.} & \sum_{a} q_1(s,a) = \rho(s) \ \forall s\\
&\sum_a q_{h+1}(s,a) = \sum_{s',a} q_h(s',a) P(s',a,s) \\ &\qquad \qquad \ \forall s, \ \forall h \in [H-1]\\
&q \ge 0
\end{split}
\end{align}
Here the variable $q_h(s,a)$ should be interpreted as the probability that the policy visits state $s$ at time-step $h$ and takes action $a$. We will refer to the variables $\{q_h(s,a)\}_{a \in \calA, s \in \calS, h \in [H]}$ as a state-action occupancy measure.
Once we solve the LP~\ref{eq:utilitarian-LP}, we can obtain the optimal policy as
\begin{align} \label{eq:q-to-pi}
 \pi_h(a|s) = \left\{ \begin{array}{cc}
\frac{q_h(s,a)  }{\sum_b q_h(s,b)} & \textrm{ if } \sum_b q_h(s,b) > 0\\
\frac{1 }{A} & \textrm{o.w.}
 \end{array}\right.
\end{align}

\textbf{Occupancy Measure Polytope}: Given an initial state distribution $\rho$ and a probability transition function $P$ we will write $\mathcal{Q}(\rho, P)$ to denote the set of state-action occupancy measures satisfying the Bellman flow conditions with respect to $\rho$ and $P$, i.e.
\begin{align*}
&\calQ(\rho, P) = \left\{q \ge 0: \sum_{a} q_1(s,a) = \rho(s) \ \forall s \ \textrm{and}\right.\\
&\left.\sum_a q_{h+1}(s,a) = \sum_{s',a} q_h(s',a) P(s',a,s) \forall s, \ \forall h \in [H-1] \right\}
\end{align*}

\paragraph{Utilitarian.} We maximize the sum of value functions across the $n$ agents.
$$
\pi^\star_\UL \in \argmax_\pi \sum_i V^\pi(\rho; r_i)
$$
We can solve optimization problem~\eqref{eq:dual-LP-RL} with reward function $r = \sum_i r_i$ and obtain $q^\star_\UL$.
\begin{align}
    \begin{split}
        \label{eq:utilitarian-LP}
        q^\star_\UL \in &\argmax_q \sum_{h=1}^H \sum_{s,a} q_h(s,a) \sum_i r_i(s,a)\\
        \textrm{s.t.}&\ q \in \calQ(\rho, P)
    \end{split}
\end{align}
Once we obtain $q^\star_\UL$ we can obtain the optimal utilitarian policy $\pi^\star_\UL$ through \cref{eq:q-to-pi}.

\paragraph{Max-Min Fair.} 
We maximize the worst value function across the $n$ agents.
$$
\pi^\star_\MW \in \argmax_\pi \min_{i \in [n]} V^\pi\left(\rho; r_i \right)
$$
The above optimization problem can also be solved through a linear program.
\begin{align}\label{eq:max-min-LP}
\begin{split}
\max_{q,t}\ &t\\
\textrm{s.t.} &\sum_{h=1}^H \sum_{s,a} q_h(s,a)  r_i(s,a) \ge t \ \forall i \in [n]\\
 & q \in \calQ(\rho, P) %\sum_{a} q_1(s,a) = \rho(s) \ \forall s\\
%&\sum_a q_{h+1}(s,a) = \sum_{s'} \sum_a q_h(s',a) P(s',a,s)\\
%&\qquad \qquad \forall s, \ \forall h \in [H-1]\\
%&q \ge 0
\end{split}
\end{align}

\paragraph{Nash Social Welfare.} 
We maximize the product of the value functions of the $n$ agents.
$$
\pi^\star_\NW \in \argmax_\pi \prod_{i=1}^n V^\pi(\rho; r_i)
$$
The above optimization problem can be solved through a concave optimization problem.
\begin{align}\label{eq:nash-welfare}
\begin{split}
\max_q &\sum_{i=1}^n \log \left(\sum_{h=1}^H \sum_{s,a} q_h(s,a) r_i(s,a) \right)\\
\textrm{s.t.} &\  q \in \calQ(\rho, P) 
\end{split}
\end{align}

\paragraph{Generalized Gini Social Welfare}(GGW): This notion of fairness generalizes max-min fairness and has been considered previously in the literature on multi-objective Markov decision processes~\cite{SWZ20,ZGSW21}. We are given a vector of weights $w \in \R^n$ so that $w_i \ge 0$ for each $i$, $\sum_i w_i = 1$, and $w_1 \ge w_2 \ge \ldots \ge w_n$.
 Given a policy $\pi$, let $i_1,\ldots,i_n$ be an ordering of the $n$ agents so that $V^\pi(r_{i_1}) \le V^\pi(r_{i_2}) \le \ldots \le V^\pi(r_{i_n})$. Then we maximize the following objective
 $$
 \pi^\star_\GGW \in \argmax_\pi \sum_{j=1}^n w_j V^\pi(r_{i_j})
 $$
 The optimal GGW policy can also be computed through a linear program.
\begin{align}\label{eq:GGW-LP}
\begin{split}
\max_{q,t}\ &t\\
\textrm{s.t.} &\sum_i w_i \sum_{h=1}^H \sum_{s,a} q_h(s,a)  r_{\sigma(i)}(s,a) \ge t \ \forall \sigma \in \calS_n\footnotemark \\
& q \in \calQ(\rho, P) 
\end{split}
\end{align}
\footnotetext{$\calS_n$ denotes the set of all permutations of the set $\set{1,\ldots,n}$.}

Even though the linear program above has exponentially many constraints, one can easily solve the LP using ellipsoid method. In order to apply the ellipsoid method we just need to show that there exists an efficient separation oracle. Given a candidate solution $q,t$ the bellman flow constraints i.e. the last three constraints can be checked efficiently. Then for each agent $i$ we compute $V_i(q) = \sum_{h=1}^H \sum_{s,a} q_h(s,a) r_i(s,a)$. Let $i_1, i_2,\ldots, i_n$ be an ordering of the agents so that $V_{i_1}(q) \ge V_{i_2}(q) \ge \ldots \ge V_{i_n}(q)$. Then we check if $\sum_j w_j V_{i_j}(q) \ge t$ or not. If this constraint is violated then we have a violating permutation $\sigma$ where $\sigma(j) = i_j$. On the other hand, if this constraint is satisfied then all the other constraints are satisfied since for any other permutation $\sigma$ we have $\sum_j w_j V_{\sigma(j)}(q)  \ge \sum_j w_j V_{i_j}(q) $. This is because the weights are arranged in a non-increasing order.

\section{Proof of Theorem \ref{thm:nw-to-max-min}}

% \begin{proof}
The optimal policy according to Nash social welfare can be computed through the following optimization problem.
\begin{align*}
    \max_q &\quad \prod_{i=1}^n \left(\sum_{h=1}^H \sum_{s,a} q_h(s,a) r_i(s,a) \right)\\
    \textrm{s.t.} &\quad q \in \calQ(\rho, P)
\end{align*}
Let $q^\star$ be the optimal solution of the above optimization problem. We now consider two cases.

\textbf{Case 1.} {$\NW(q^\star) > 0$}: Since $q^\star$ is a stationary point and the set $\calQ(\rho, P)$ is convex, the first order stationarity conditions gives us
\begin{equation}\label{eq:first-order-stationarity}
\nabla \NW(q^\star)^\top (q - q^\star) \le 0 \quad \forall q \in \calQ(\rho, P)
\end{equation}
The derivative of the Nash welfare objective with respect to $q_h(s,a)$ is the following.
\begin{align*}
\frac{\partial \NW(q)}{\partial q_h(s,a)} &= \sum_i \prod_{j \neq i} \left(\sum_{h=1}^H \sum_{s,a} q_h(s,a) r_i(s,a) \right) r_i(s,a)\\
&= \NW(q) \sum_i \frac{r_i(s,a)}{V_i(q)}
\end{align*}
Then equation~\eqref{eq:first-order-stationarity} gives us the following result.
\begin{align*}
    &\NW(q^\star) \sum_i \sum_{h,s,a} \frac{r_i(s,a)}{V_i(q^\star)} \left(q_h(s,a) - q_h^\star(s,a) \right) \le 0\\
    \Rightarrow &\NW(q^\star) \sum_i \frac{V_i(q) - V_i(q^\star)}{V_i(q^\star)} \le 0\\
    \Rightarrow & \sum_i \frac{V_i(q)}{V_i(q^\star)} \le n
\end{align*}
This immediately implies that $\max_i \frac{V_i(q)}{V_i(q^\star)} \le n$.

\textbf{Case 2.} $\NW(q^\star) = 0$:
This implies that for every policy $\pi$, $V^{\pi}(r_i) = 0$ for some $i$.
Hence, the minimum value of every policy $\pi$ is zero and the MM social welfare is zero, contradicting the assumption of the theorem.
% \jiarui{Should this be "for some $i$"?}
% Moreover, in this case, $\NW(\pi^\star_\MW) \le \NW(q^\star) = 0$. Since rewards are non-negative in our setting we get $V_i(\pi^\star_\MW) = 0$. Therefore, the lower bound is trivially satisfied.

% \end{proof}

\paragraph{Upper Bound}
% \deb{TODO: add upper bound instance.}
We now construct an instance for the upper bound on the value functions under $\pi^\star_\NW$. Given any $\delta > 0$ we construct the following instance.
Consider the following example with a single-state MDP and two actions $a$ and $b$ available.
The rewards are
\begin{align*}
r_1(s,a) &= \delta^n, 
\quad
r_i(s,a) = 1 \text{ for all }i \neq 1, \\ 
\text{and}\quad
r_i(s,b) &= \delta \text{ for all } i.
\end{align*}
Clearly, $\pi^b$, which selects action $b$ with probability $1$ maximizes the minimum value, whereby $V^{\pi^b}(r_i) = \delta$ for all $i \in [n]$.
To find out $\pi_{\NW}^\star$, consider an arbitrary policy which plays action $a$ with probability $x$, and we optimize $x$ with respect to the Nash welfare.
Consider $\NW$ as a function of $x$, we have
\[
\frac{\partial \log\NW(x)}{\partial x} =
(n-1) \log (x + (1-x) \delta) +
\log \left(x \delta^n + (1-x) \delta \right),
\]
of which the zero is given by
\[
\frac{(1-\delta)(n-1)}{x + (1-x) \delta} = \frac{1 - \delta^{n-1}}{x \delta^{n-1} + (1-x)}.
\]
Rearranging the terms we get $x = (n-1)/n + o(1)$, where $o(1)$ is with respect to $\delta$.
Hence,
\[
\frac{V^{\pi_{\NW}^\star}(r_1)}{V^{\pi_{\MW}^\star}(r_1)}
= \frac{\frac{n-1}{n} \delta^n + \frac{1}{n}\delta + o(1)}{\delta}
= \frac{1}{n} + o(1).
\]

\section{Properties of Fairness Measures}

%We consider the following objective.
%\begin{equation}
%\MW(\pi; \rr) = \min_i V^\pi(\rho; \rr)
%\end{equation}
\subsection{Mininum Welfare ($\MW$)}
\paragraph{Minimum Welfare Violates ~\ref{axiom:PO}~(PO).}
Consider a MDP with a single state and two actions $a$ and $b$. Agent $1$ has the same reward for both the actions, say $1$. On the other hand, agent $2$ has reward $1$ for action $a$ and reward $2$ for action $b$. Consider a policy $\pi_a$ that always pulls action $a$ in state $s$. Then $V^{\pi_a}(r_1) = V^{\pi_a}(r_2) = H$. Now consider another policy $\pi_b$ that always pulls action $b$ in state $s$. Then $V^{\pi_b}(r_1) = H$ and $V^{\pi_b}(r_2) = 2H$. Since the minimum value function does not change we have $\MW(\pi_a) = \MW(\pi_b) = H$.

% \textbf{Violation of axiom~\ref{axiom:IIA}}:
\paragraph{Minimum Welare Violates ~\ref{axiom:IIA}~(IIAN).}
We again consider a MDP with a single state and two actions $a$ and $b$. We consider two reward functions $\rr$ and $\tilde{\rr}$ defined as follows. For the first agent 
% $r_1(s,a) = 1, r_1(s,b) = 0$ and $\tilde{r}_1(s,a) = 1, \tilde{r}_1(s,b) = 0$.
\begin{equation}\label{eq:reward-agent-1}
r_1(s,a) = \tilde{r}_1(s,a) =1, \text{ and } r_1(s,b) = \tilde{r}_1(s,b) = 0.
\end{equation}
On the other hand, for the second agent, we have 
% $r_2(s,a) = 1/4, r_2(s,b) = 3/4$ and $\tilde{r}_2(s,a) = 1, \tilde{r}_2(s,b) = 3$. 
\begin{equation}\label{eq:reward-agent-2}
r_2(s,a) = \frac{1}{4}, r_2(s,b) = \frac{3}{4}, \tilde{r}_2(s,a) = 1, \text{ and } \tilde{r}_2(s,b) = 3.
\end{equation}
Consider two policies: $\pi_1$ pulls $a$ or $b$ with equal probability, whereas $\pi_2$ pulls $a$ with probability $3/4$ and $b$ with probability $1/4$. For the first agent we then have,
\[
\textstyle
% V^{\pi_1}(r_1) = \frac{H}{2}, V^{\pi_2}(r_1) = \frac{3H}{4}, V^{{\pi_1}}(\tilde{r}_1) = \frac{H}{2}, V^{\pi_2}(\tilde{r}_1) = \frac{3H}{4} 
V^{\pi_1}(r_1) = V^{{\pi_1}}(\tilde{r}_1) = \frac{H}{2}, \text{ and } V^{\pi_2}(r_1) = V^{\pi_2}(\tilde{r}_1) = \frac{3H}{4}.
\]
For the second agent we have,
\[
\textstyle
V^{\pi_1}(r_2) = \frac{H}{2}, V^{\pi_2}(r_2) = \frac{3H}{8}, V^{{\pi_1}}(\tilde{r}_2) = 2H, V^{\pi_2}(\tilde{r}_2) = \frac{3H}{2}.
\]
Therefore, we observe that the rewards and the policy satisfy the following condition.
$$
\frac{V^{\pi_1}(r_1)}{V^{\pi_2}(r_1)} = \frac{V^{\pi_1}(\tilde{r}_1)}{V^{\pi_2}(\tilde{r}_1)} = \frac{2}{3}\ \textrm{ and } \frac{V^{\pi_1}(r_2)}{V^{\pi_2}(r_2)} = \frac{V^{\pi_1}(\tilde{r}_2)}{V^{\pi_2}(\tilde{r}_2)} = \frac{4}{3}
$$
For the reward vector $\rr$ we have $\MW(\pi_1; \rr) = \frac{H}{2}$ and $\MW(\pi_2; \rr) = \frac{3H}{8}$. 
On the other hand, for the reward vector $\tilde{\rr}$ we have $\MW(\pi_1; \tilde{\rr}) = \frac{H}{2}$ and $\MW(\pi_2; \tilde{\rr}) = \frac{3H}{4}$. Therefore, $\MW(\pi_1;\rr) > \MW(\pi_2; \rr)$ but $\MW(\pi_1; \tilde{\rr}) < \MW(\pi_2; \tilde{\rr})$. This implies that the objective $\MW$ does not satisfy the independence of irrelevant alternatives.

\subsection{Generalized Gini Welfare ($\GGW$)}
\paragraph{GGW Violates ~\ref{axiom:IIA}~(IIA).}
%\jiarui{This part can be put in the appendix if we need more space}
We consider the same example as discussed for the setting of max-min fairness. Consider a weight vector $\mathbf{w} = (w_1, w_2)$ where $w_1 > w_2 > 0$ and $w_1 + w_2 = 1$. For the reward vector $\rr$, we have 
% the following values of generalized Gini welfare.
\begin{align*}
\textstyle
\GGW(\pi_1, \rr) &= \frac{H}{2}, \\ 
\textrm{and} \quad
\GGW(\pi_2, \rr) &= \frac{3H}{8} w_1 + \frac{3H}{4} w_2 = \frac{3H}{8} (1 + w_2).
\end{align*}
For the reward vector $\tilde{\rr}$, we have
\begin{align*}
\textstyle
\GGW(\pi_1; \tilde{\rr}) &= \frac{H}{2} w_1 + 2H w_2, \\ 
\textrm{and} \quad
\GGW(\pi_2; \tilde{\rr}) &= \frac{3H}{4} w_1 + \frac{3H}{2} w_2.
\end{align*}
We now consider two cases. First, if $w_2 < \frac{1}{3}$ we have $\GGW(\pi_1; \rr) > \GGW(\pi_2; \rr)$ but $\GGW(\pi_1; \tilde{\rr}) < \GGW(\pi_2; \tilde{\rr})$. On the other hand, if $w_2 > \frac{1}{3}$ we have $\GGW(\pi_1; \rr) < \GGW(\pi_2; \rr)$ but $\GGW(\pi_1; \tilde{\rr}) > \GGW(\pi_2; \tilde{\rr})$. 
%\jiarui{I guess $\tpi_1, \tpi_2$ should be $\pi_1, \pi_2$?}

Now we consider the remaining case of $w_2 = \frac{1}{3}$. We now consider a different choice of reward functions than defined in equations~\eqref{eq:reward-agent-1} and \eqref{eq:reward-agent-2}. Reward functions $r_1$ and $\tilde{r}_1$ remain as they are, but we change the definition of reward functions $r_2$ and $\tilde{r_2}$ as follows.
$$
r_2(s,a) = \frac{1}{2}, r_2(s,b) = \frac{2}{3}, \ \textrm{and}\ \tilde{r}_2(s,a) = 1, \tilde{r}_2(s,b) = \frac{4}{3}
$$
Now it can be again checked that the required assumptions for axiom~\ref{axiom:IIA} is satisfied. However, for the weight vector $(2/3,1/3)$ we have $\GGW(\pi_1;\rr) > \GGW(\pi_2;\rr)$ but $\GGW(\pi_1; \tilde{\rr}) < \GGW(\pi_2; \tilde{\rr})$.
%\jiarui{maybe just show IIAN is violated for some $w$, so we can just write one sentence after showing Minimum Value violates IIAN, saying that GGW also violates IIAN according to the same example}
Therefore, we have shown that for any weight vector with $w_1 > w_2 > 0$, one can construct an instance where the fairness measure $\GGW$ violates axiom IIA. Note that, the case of $w_1 = 1, w_2 = 0$ corresponds to minimum welfare, and the construction can be generalized to more than two agents by just duplicating the reward functions.

\begin{proposition}
$\MW$ satisfies axioms \ref{axiom:ANON} and \ref{axiom:CON}.
\end{proposition}
\begin{proof}
Since  minimum is an anonymous function, the objective $\MW$ satisfies axiom~\ref{axiom:ANON}. We now check axiom~\ref{axiom:CON}. Given a policy $\pi$, let $q^\pi$ be the occupancy measure of policy $\pi$. Then we can express minimum value function as $\MW(\pi; {\rr}) = \min_i V^\pi(\rho; r_i) = \min_i \sum_{h,s,a} q_h^\pi(s,a) r_i(s,a)$. We can write down $\MW$ as a function of the occupancy measure of a policy i.e.
$$
\MW_o(q; {\rr}) = \min_i \sum_{h,s,a} q_h(s,a) r_i(s,a)
$$
Since $\MW_o(\cdot; {\rr})$ is a minimum of linear functions $\MW_o$ is continuous. Therefore, there exists a point on the line segment joining $q_1$ and $q_3$ that achieves the value $\MW_o(q_2; {\rr})$. In particular, there exists $\alpha \in [0,1]$ so that
\[
\MW_o(\alpha q_1 + (1-\alpha) q_3; {\rr}) = \MW_o(q_2;{\rr}) = \MW(\pi_2; {\rr}) \qedhere
\]
\end{proof}

\begin{proposition}
There exists a weight vector $w \in \R^n_+$ so that $\GGW$ satisfies axioms \ref{axiom:PO}, \ref{axiom:ANON} and \ref{axiom:CON}.
\end{proposition}
\begin{proof}
Consider any weight vector $w$ so that $w_1 > w_2 > \ldots > w_n > 0$ and $\sum_j w_j = 1$. Given a policy $\pi$, the ordering of the agents in terms of their value functions remains unchanged even if we permute the agents. This implies hat the objective $\GGW$ is anonymous for any weight vector and it satisfies axiom~\ref{axiom:ANON}.  We now check axiom ~\ref{axiom:CON}. Given a policy $\pi$, let $q^\pi$ be the occupancy measure of policy $\pi$. Then we can express generalized Gini welfare as a function of the occupancy measure of a policy i.e.
$$
\GGW_o(q; \{r_i\}_{i \in [n]}) = \sum_j w_j \sum_{h,s,a} q_h(s,a) r_{i_j}(s,a)
$$
Note that $\GGW_o(\cdot; \{r_i\}_{i \in [n]})$ is a weighted sum of order statistics. Since each order statistic is a continuous function of the occupancy measure $q$, $\GGW_o$ is a continuous function. Therefore, there exists a point on the line segment joining $q_1$ and $q_3$ that achieves the value $\GGW_o(q_2; \{r_i\}_{i \in [n]})$. In particular, there exists $\alpha \in [0,1]$ so that
$$
\GGW_o(\alpha q_1 + (1-\alpha) q_3; \{r_i\}_{i \in [n]}) = \GGW_o(q_2; \{r_i\}_{i \in [n]})
$$
Now any policy $\pi^{q_2}$ achieves the generalized Gini welfare under the occupancy measure $q_2$.

In order to check axiom~\ref{axiom:PO} we assume $V^\pi(r_i) \ge V^{\tpi}(r_i)$ for all $i$ and there exists $j$ such that $V^\pi(r_j) > V^{\tpi}(r_j)$. Without loss of generality we can also assume that the agents are ordered so that $V^{\tpi}(r_1) \ge V^{\tpi}(r_2) \ge \ldots \ge V^{\tpi}(r_n)$. Then we have $V^\pi(r_{i_k}) \ge V^\pi(r_k) \ge V^{\tpi}(r_k)$ and $V^\pi(r_{i_j}) \ge V^\pi(r_j) > V^{\tpi}(r_j)$. Since $w_k > 0$ we have $\GGW(\pi; \{r_i\}_{i \in [n]}) > \GGW(\pi; \{r_i\}_{i \in [n]})$.
\end{proof}

% \paragraph{Violation of Axiom~\ref{axiom:IIA}}: We consider the same example as discussed for the setting of max-min fairness. Since there are two agents we will consider a weight vector $w = (w_1, w_2)$ where $w_1 > w_2 > 0$ and $w_1 + w_2 = 1$. For the reward vector $r$ we have the following values of generalized Gini welfare.
% $$
% \GGW(\pi_1,r) = \frac{H}{2} \ \textrm{and}\ \GGW(\pi_2, r) = \frac{3H}{8} w_1 + \frac{3H}{4} w_2 = \frac{3H}{8} (1 + w_2)
% $$
% For the reward vector $\tilde{r}$ we have the following values of generalized Gini welfare.
% $$
% \GGW(\pi_1; \tilde{r}) = \frac{H}{2} w_1 + 2H w_2 \ \textrm{and}\ \GGW(\pi_2; \tilde{r}) = \frac{3H}{4} w_1 + \frac{3H}{2} w_2
% $$
% We now consider two cases. First, if $w_2 < \frac{1}{3}$ we have $\GGW(\pi_1; r) > \GGW(\pi_2; r)$ but $\GGW(\tpi_1; \tilde{r}) < \GGW(\tpi_2; \tilde{r})$. On the other hand, if $w_2 > \frac{1}{3}$ we have $\GGW(\pi_1; r) < \GGW(\pi_2; r)$ but $\GGW(\tpi_1; \tilde{r}) > \GGW(\tpi_2; \tilde{r})$. 

% Now we consider the remaining case of $w_2 = \frac{1}{3}$. 
\begin{proposition}
$\NW$ satisfies Axioms~\ref{axiom:PO}, \ref{axiom:IIA}, \ref{axiom:ANON}, and \ref{axiom:CON}.
\end{proposition}
\begin{proof}
Since $\NW$ is a monotonically increasing function of $V^\pi(r_i)$, axiom \ref{axiom:PO} is immediately satisfied. Moreover, the function $\NW$ is a product of the values across the agents and hence anonymous. In order to check axiom \ref{axiom:IIA} suppose there exist policies $\pi_1, \pi_2, \tilde{\pi}_1, \tilde{\pi}_2$ and reward vectors $r, \tilde{r}$ so that for all $i$ we have
$$
\frac{V^{\pi_1}(\rho;r_i)}{V^{\pi_2}(\rho; r_i)} = \frac{V^{\tpi_1}(\rho; \tilde{r}_i)}{V^{\tpi_2}(\rho; \tilde{r}_i)}
$$
Now suppose $\NW(\pi_1; {\rr}) \ge \NW(\pi_2; {\rr})$. This implies $\prod_i V^{\pi_1}(\rho; r_i) \ge \prod_i V^{\pi_2}(\rho; r_i)$. This gives us the following inequality.
$$
\frac{\prod_i V^{\tilde{\pi}_1}(\rho; \tilde{r}_i)}{\prod_i V^{\tilde{\pi}_2}(\rho; \tilde{r}_i)} = \frac{\prod_i V^{{\pi}_1}(\rho; {r}_i)}{\prod_i V^{{\pi}_2}(\rho; {r}_i)} \ge 1
$$
Therefore, we have $\NW(\tilde{\pi}_1; \{\tilde{r}_i\}_{i \in [n]}) \ge \NW(\tilde{\pi}_2; \{\tilde{r}_i\}_{i \in [n]})$. The other direction of the implication can be proved analogously.

Now we check the final axiom~\ref{axiom:CON}. Given a policy $\pi$, let $q^\pi$ be the occupancy measure of policy $\pi$. Then we can express the  Nash social welfare functional as $\NW(\pi; \{r_i\}_{i \in [n]}) = \prod_i  V^\pi(\rho; r_i) = \prod_i \left(\sum_{h,s,a} q^\pi_h(s,a) r_i(s,a) \right)$. Therefore, we can write down log-Nash social welfare as a function of the occupancy measure of a policy i.e. 
$$
\NW_o(q; \rr) = \prod_i  \left(\sum_{h,s,a} q_h(s,a) r_i(s,a) \right)
$$
$\NW_o$ is a continuous function of the occupancy measure. We are also given that $NW_o(q_1; \rr) \ge \NW_o(q_2; \rr) \ge \NW_o(q_3; \rr)$. Therefore, there exists a point on the line segment joining $q_1$ and $q_3$ that achieves the value $\NW_o(q_2; \rr)$. In particular, there exists $\alpha \in[0,1]$ so that
\begin{align*}
\NW_o(\alpha q_1 + (1-\alpha) q_3; \rr) 
&= \NW_o(q_2; \rr) \\
&= \NW(\pi_2; \rr)
\end{align*}
Now any policy with state-action occupancy measure $\alpha q_1 + (1-\alpha) q_3$ achieves the intermediate value $\NW(\pi_2; \rr)$.
\end{proof}

\section{Proof of Theorem~\ref{thm:NW-unique}} 

We prove that Axioms~\ref{axiom:PO}, \ref{axiom:IIA}, \ref{axiom:ANON}, and \ref{axiom:CON} uniquely characterize the Nash Social Welfare function. 
% We make the following assumptions. \deb{need to update the assumptions.}
% \begin{enumerate}
% \item There exists an action $0$ available at all states and gives a reward of zero.
% \item There are at least three actions available at each state excluding the action $0$.
% \end{enumerate}
Throughout we will fix the transition probability function of the MDP and let different agents pick different reward functions. 
We assume that there are at least four actions available at each state. 
%For simplicity, we also assume that $H=1$; the proof can be easily generalized to any arbitrary $H$.
% \deb{We probably should prove for general $H$. $H=1$ implies contextual bandit i.e. transition probability function $P$ is not important.} \deb{I am stating the proof for general $H$. This is where axiom~\ref{axiom:CON} is different than standard continuity assumption.}

\textbf{Defining Partial Order $\succcurlyeq$}: We first define a partial order $\succcurlyeq$ over the elements in $\R^n_+ = \set{x \in \R^n: x_i > 0 \ \forall i}$. %Given $x,y \in \R^n_+$ let us define the following reward functions
\begin{equation}\label{eq:defn-order}
x \succcurlyeq y \iff W\left(\pi^a; {\rr}\right) \ge W\left(\pi^b; {\rr} \right)
\end{equation}
for arbitrary $a,b \in A$, where $\pi^a$ (resp. $\pi^b$) denotes a policy such that $\pi^a(s) = a$ (resp. $\pi^b(s) = b$), and $\rr$ is such that 
\begin{equation}\label{eq:reward-a-b}
r_i(s,a) = x_i
\ \text{ and }\
r_i(s,b) = y_i
%\ \text{ and }\
%r_i(s', c) = 0
% \quad \forall i \in [n].
\end{equation}
for all $i \in [n]$,
and all the other entries of $\rr$ are arbitrary.
%The other entries of the reward vector $\rr$ can be arbitrary.
Note that because of Axiom~\ref{axiom:IIA} this order $\succcurlyeq$ is well-defined despite the arbitrary choice of $a$ and $b$ and reward vector $\rr$ as 
we have $V^{\pi^a}(r_i) = x_i\cdot H$ and $V^{\pi^b}(r_i) = y_i\cdot H$.

% Suppose we pick another pair of alternatives $c$ and $d$ satisfying \cref{eq:reward-a-b}, i.e.,
% \begin{equation}
% \tilde{r}_i(s,c) = x_i, \ \tilde{r}_i(s,d) = y_i \quad \forall i \in [n]
% \end{equation}
% Then we have $V^{\pi^a}(r_i) = V^{\pi^c}(\tilde{r}_i) = H x_i$ and $V^{\pi^b}(r_i) = V^{\pi^d}(\tilde{r}_i) = H y_i$. Therefore, by Axiom~\ref{axiom:IIA} we must have $W\left(\pi^c; {\rr}\right) \ge W\left(\pi^d; {\rr} \right)$ \jiarui{if and only if $x \succcurlyeq y$?}. 

We now verify that the relation $\succcurlyeq$ satisfies {\em reflexivity}, {\em symmetry}, and {\em transitivity}. Reflexivity and symmetry immediately follows from the definition of the reward function in \eqref{eq:reward-a-b}. For transitivity, suppose we are given $x \succcurlyeq y$ and $y \succcurlyeq z$. We construct a reward function $\rr$ so that
$$
r_i(s,a) = x_i, r_i(s,b) = y_i, \textrm{ and } r_i(s,c) = z_i
$$
Then we have $W\left(\pi^a; {\rr}\right) \ge W\left(\pi^b; {\rr} \right)$ and $W\left(\pi^b; {\rr}\right) \ge W\left(\pi^c; {\rr} \right)$. 
This implies $W\left(\pi^a; {\rr}\right) \ge W\left(\pi^c; {\rr} \right)$ and $x \succcurlyeq z$. 
Therefore, any two elements of $\R^n_+$ are comparable under the relation $\succcurlyeq$.

\begin{lemma}\label{lem:strict-order}
Suppose $x, y \in \R^n_+$ with $x_i \ge y_i$ for all $i$ and $x_i > y_i$ for some $i$. Then $x \succ y$.
\end{lemma}
\begin{proof}
Let $\rr$ be such that $r_i(s,a) = x_i$ and $r_i(s,b) = y_i$ for all $i \in [n]$.
Note that we have $V^{\pi^a}(r_i) \ge V^{\pi^b}(r_i)$ for all $i \in [n]$ and $V^{\pi^a}(r_i) > V^{\pi^b}(r_i)$ for some individual $i$. Therefore, by Axiom~\ref{axiom:PO} we must have $W\left(\pi^a; {\rr} \right) > W\left(\pi^b; {\rr} \right)$. By definition~\eqref{eq:defn-order}, this implies that $x \succ y$.
\end{proof}

\paragraph{Defining Function $F$.} 
Let $e \in \R^n_+$ be a vector with all component being $1$. 
For any $x \in \R^n_+$, there exists $\lambda_1 > \lambda_0 > 0$ such that $\lambda_0  \ee < x < \lambda_1 \ee$ (where the comparison is component-wise). 
% We choose the following reward functions.
We choose a reward function $\rr$ such that
\[
r_i(s,a) = \lambda_1,\ r_i(s,b) = \lambda_0, \ r_i(s,c) = x_i,
\]
and $r_i(s,d)$ is arbitrary.
Then we have $W\left(\pi^a; {\rr}\right) > W\left(\pi^c; {\rr} \right) > W\left(\pi^b; {\rr} \right)$. By Axiom~\ref{axiom:CON} there exists $\alpha \in [0,1]$ such that 
\begin{equation}
\label{eq:F-W-CON}
W\left(\pi^c; {\rr} \right) = W\left(\pi^{\alpha q_a + (1-\alpha) q_b} ; \rr \right)
\end{equation}
where $q_a = q^{\pi^a}$ and $q_b = q^{\pi^b}$. 
Under policy $\pi^c$ the value function of agent $i$ is $V^{\pi^c}(r_i) = x_i \cdot H$. And the value function under the policy $\pi_\alpha  = \pi^{\alpha q_a + (1-\alpha) q_b}$ is 
\begin{align*}
V^{\pi_\alpha}(r_i) 
&= \left \langle \alpha q_a + (1-\alpha) q_b, r_i \right \rangle \\
&= \alpha \left \langle q_a, r_i \right \rangle + (1-\alpha) \left \langle q_b, r_i \right \rangle  \\
&= \alpha \lambda_1\cdot H + (1-\alpha) \lambda_0 \cdot H. 
\end{align*}
Therefore, according to the definition of the order $\succcurlyeq$ we have $x \sim [\alpha \lambda_1 + (1-\alpha) \lambda_0]\ee$ and we define 
\[
F(x) = \alpha \lambda_1 + (1-\alpha) \lambda_0.
\]
% Note that in this definition we implicitly apply a function that picks an $\alpha$ that makes \eqref{eq:F-W-CON} hold for each $x$.
We first verify that $F$ is a well-defined function despite the arbitrary choice of $\lambda_0, \lambda_1$, and $\alpha$. 
Indeed, suppose that there exists $\tilde{\lambda}_1 > \tilde{\lambda}_0 > 0$ and $\tilde{\alpha}$ such that $\tilde{\lambda}_0 \ee < x < \tilde{\lambda}_1 \ee$ and $\tilde{\alpha}$ satisfies  \eqref{eq:F-W-CON} with respect to $\lambda_0$ and $\lambda_1$.
% Then we can again proceed as above and obtain $\tilde{\alpha}$. 
We have $[\alpha \lambda_1 + (1-\alpha) \lambda_0]\ee \sim [\tilde{\alpha} \tilde{\lambda_1} + (1-\tilde{\alpha}) \tilde{\lambda_0}]\ee$. Then by Lemma \ref{lem:strict-order} it must be that $\alpha \lambda_1 + (1-\alpha) \lambda_0 = \tilde{\alpha} \tilde{\lambda_1} + (1-\tilde{\alpha}) \tilde{\lambda_0}$. Hence, the value of $F(x)$ is independent of the choice of $\lambda_0, \lambda_1$, and $\alpha$.

We now establish an important property of the function $F$. The proof of this claim is similar to the proof in \cite{KN79} and uses the following lemma by \citet{osborne1976irrelevant} (also see Lemma~3.5 stated by \citet{KN79}). We present our version of the proof for completeness.

\begin{lemma}[\citet{osborne1976irrelevant}]
\label{lmm:osborne}
Suppose $\succcurlyeq$ is reflexive, symmetric, and transitive, and $F: \mathbb{R}_+^n \to \mathbb{R}$ satisfies the following properties for any $x,y \in \mathbb{R}_+^n$:
\begin{itemize}
\item[(i)] 
if $x \succcurlyeq y$, then $F(x) \ge F(y)$; and

\item[(ii)] 
$F(x) \ge F(y)$ if and only if $F(\delta_1 x_1, \dots, \delta_n x_n) \ge F(\delta_1 y_1, \dots, \delta_n y_n)$ for all $\delta \in \mathbb{R}_+^n$.
\end{itemize}
Then there exist non-negative real constants $c_1,\dots,c_n$ and a monotone increasing function $g: \mathbb{R} \to \mathbb{R}$, such that $F(x) \equiv g\left( \prod_{i=1}^n x_i^{c_i} \right)$. 
\end{lemma}

\begin{lemma}\label{lem:characterization-F}
% There exists a positive real number $c$ and a monotone increasing function $V$ such that $F(x) = V\left(\prod_{i=1}^n x_i^{c}\right)$.
There exists a monotone increasing function $g$ such that $F(x) = g\left(\prod_{i=1}^n x_i\right)$.
\end{lemma}
\begin{proof}
We show that the function $F$ satisfies two requirements of Lemma~\ref{lmm:osborne}.
% \cite{osborne1976irrelevant} (also see Lemma~3.5 stated by \citet{KN79}):
% \begin{itemize}
% \item if $x \succcurlyeq y$, then $F(x) \ge F(y)$; and
% \item $F(x) \ge F(y)$ if and only if $f(\lambda_1 x_1, \dots, \lambda_n x_n) \ge f(\lambda_1 y_1, \dots, \lambda_n y_n)$ for all $\lambda \in \mathbb{R}_+^n$.
% \end{itemize}
This will imply that there exist non-negative real constants $c_1,\dots,c_n$ and a monotone increasing function $g: \mathbb{R} \to \mathbb{R}$, such that $F(x) \equiv g\left( \prod_{i=1}^n x_i^{c_i} \right)$. 
By Axiom~\ref{axiom:ANON}, the constants $c_1,\dots,c_n$ must be identical, so the statement of the lemma follows. 

\paragraph{Requirement (i).}
First, given $x \succcurlyeq y$, we need to show that $F(x) \ge F(y)$. Since $x, y \in \R^n_+$ there exists $\lambda_0, \lambda_1 > 0$ so that $\lambda_0 \ee < x < \lambda_1 \ee$ and $\lambda_0 \ee < y < \lambda_1 \ee$. We choose a reward function $\rr$ such that
$$
r_i(s,a) = \lambda_1, r_i(s,b) = \lambda_0, r_i(s,c) = x_i, r_i(s,d) = y_i
$$
and $r_i = 0$ otherwise. There exists $\alpha$ such that
$$
F(x) = \alpha \lambda_1 + (1-\alpha) \lambda_0 \ \textrm{and}\ W(\pi^c ; \rr) = W\left(\pi^{\alpha q_a + (1-\alpha) q_b}; \rr \right).
$$

There also exists $\beta$ such that
$$
F(y) = \beta \lambda_1 + (1-\beta) \lambda_0 \ \textrm{and}\ W(\pi^c ; \rr') = W\left(\pi^{\beta q_a + (1-\beta) q_b}; \rr' \right),
$$
where $r'_i(s,a) = \lambda_1, r'_i(s,b) = \lambda_0, r'_i(s,c) = y_i$ and other entries of $\rr'$ are arbitrary.
We have 
\[
V^{\pi^d}(\rr)/V^{\pi^{\beta q_a + (1-\beta) q_b}}(\rr) 
= 
V^{\pi^c}(\rr')/V^{\pi^{\beta q_a + (1-\beta) q_b}}(\rr'). 
\]
By Axiom~\ref{axiom:IIA}, this means that $W(\pi^d ; \rr) \ge W\left(\pi^{\beta q_a + (1-\beta) q_b}; \rr  \right)$ if and only if $W(\pi^c ; \rr') \ge W\left(\pi^{\beta q_a + (1-\beta) q_b}; \rr'  \right)$.
Hence, we also get that
\begin{equation}
\label{eq:W-pi-d}
W(\pi^d ; \rr) = W\left(\pi^{\beta q_a + (1-\beta) q_b}; \rr \right).
\end{equation}

Since $x\succcurlyeq y$ from definition \eqref{eq:defn-order} we have $W(\pi^c; \rr) \ge W(\pi^d; \rr)$. This implies that 
\begin{equation}
\label{eq:W-x-y}
 W\left(\pi^{\alpha q_a + (1-\alpha) q_b}; \rr \right) \ge W\left(\pi^{\beta q_a + (1-\beta) q_b}; \rr  \right).
\end{equation}
But the vector of value functions under policy $\pi^{\alpha q_a + (1-\alpha) q_b}$ is $(\alpha \lambda_1 + (1-\alpha) \lambda_0)\cdot H \ee$ and similarly the vector of value functions under policy $\pi^{\beta q_a + (1-\beta) q_b} $ is $(\beta \lambda_1 + (1-\beta) \lambda_0)\cdot H \ee$.
Therefore, it must be that 
$$
(\alpha \lambda_1 + (1-\alpha) \lambda_0) \cdot H\ee \ge (\beta \lambda_1 + (1-\beta) \lambda_0) \cdot H\ee
$$
as otherwise we would have $(\alpha \lambda_1 + (1-\alpha) \lambda_0) \cdot H\ee < (\beta \lambda_1 + (1-\beta) \lambda_0) \cdot H\ee$, which contradicts \eqref{eq:W-x-y} given Axiom~\ref{axiom:PO}.
It follows that $F(x) = \alpha \lambda_1 + (1-\alpha) \lambda_0 \ge \beta \lambda_1 + (1-\beta) \lambda_0 = F(y)$.

\paragraph{Requirement (ii).}
Next, we show that $F(x) \ge F(y)$ if and only if $F(\delta_1 x_1, \ldots, \delta_n x_n) \ge F(\delta_1 y_1, \ldots, \delta_n y_n)$ for all $x,y$ and positive real numbers $\delta_i$ for $i = 1,\ldots,n$. 
First, suppose $F(x) \ge F(y)$. Define $F(x), F(y)$ and the parameters $\alpha, \beta$ as we defined above. Since $F(x) \ge F(y)$, we have
$$
\left(\alpha \lambda_1 + (1-\alpha) \lambda_0\right) \ee \ge \left(\beta \lambda_1 + (1-\beta) \lambda_1 \right) \ee.
$$
Hence, $W(\pi^c; \rr) \ge W(\pi^d; \rr)$ by Axiom~\ref{axiom:PO}. 

We now construct a new reward vector $\tilde{\rr}$.
Let $\underline{\delta} = \min\set{\delta_1,\ldots, \delta_n}$ and $\bar{\delta} = \max \set{\delta_1, \ldots, \delta_n}$. Then we define the following reward vector $\tilde{\rr}$
\begin{align*}
&\tilde{r}_i(s,a) = \bar{\delta} \lambda_1,&& \tilde{r}_i(s,b) = \underline{\delta} \lambda_0,\\
&\tilde{r}_i(s,c) = \delta_i x_i,&& \tilde{r}_i(s,d) = \delta_i y_i.
\end{align*}
We now apply axiom \ref{axiom:IIA} with respect to policy $\pi^c, \pi^d$ and reward vectors $\rr$ and $\tilde{\rr}$. Note that
$$
\frac{V^{\pi^c}(r_i)}{V^{\pi^d}(r_i)} = \frac{V^{\pi^c}(\tilde{r}_i)}{V^{\pi^d}(\tilde{r}_i)}.
$$
Since the function satisfies Axiom~\ref{axiom:IIA}, we have $W(\pi^c; \tilde{\rr}) \ge W(\pi^d; \tilde{\rr})$. From Definition~\eqref{eq:F-W-CON} we know that there exists $\tilde{\alpha}, \tilde{\beta}$ such that 
$$
W(\pi^c; \tilde{\rr}) = W \left( \pi^{\tilde{\alpha} q_a + (1-\tilde{\alpha}) q_b}; \tilde{\rr}\right)
$$
and, similarly to \eqref{eq:W-pi-d},
$$
W(\pi^d; \tilde{\rr}) = W \left( \pi^{\tilde{\beta} q_a + (1-\tilde{\beta}) q_b}; \tilde{\rr}\right).
$$
By Axiom~\ref{axiom:PO}, this implies that
$$
\left({\tilde{\alpha} \bar{\delta} \lambda_1 + (1-\tilde{\alpha}) \underline{\delta} \lambda_0}\right) \ee \ge \left( {\tilde{\beta} \bar{\delta} \lambda_1 + (1-\tilde{\beta}) \underline{\delta} \lambda_0}\right) \ee
$$
as the two sides are the value vectors of $\pi^{\tilde{\alpha} q_a + (1-\tilde{\alpha}) q_b}$ and $\pi^{\tilde{\beta} q_a + (1-\tilde{\beta}) q_b}$.
% \jiarui{I changed $\succcurlyeq$ to $\ge$ above to avoid confusion. I think you mean a component-wise comparison between the two vectors, instead of the order defined in \eqref{eq:defn-order}?}
It follows that 
\begin{align*}
&F(\delta_1 x_1, \ldots, \delta_n x_n) 
= \tilde{\alpha} \bar{\delta} \lambda_1 + (1-\tilde{\alpha}) \underline{\delta} \lambda_0 \\
&\qquad \qquad \ge {\tilde{\beta} \bar{\delta} \lambda_1 + (1-\tilde{\beta}) \underline{\delta} \lambda_0} 
= F(\delta_1 y_1, \ldots, \delta_n y_n).
\end{align*}
This completes the proof.
\end{proof}

\paragraph{From Function $F$ to $\NW$.} 
Suppose there are two policies $\pi_1$ and $\pi_2$ such that $W\left(\pi_1; {\rr}\right) \ge W\left(\pi_2; {\rr} \right)$. Let us denote $q_1 = q^{\pi_1}$ and $q_2 = q^{\pi_2}$. Consider two vectors $u, v \in \R^n_+$ such that
\[
u_i = {\left\langle r_i, q_1\right \rangle}
\ \text{ and }\
v_i = {\left\langle r_i, q_2\right \rangle}.
\]
In order to determine whether $u \succcurlyeq v$ or $v \succcurlyeq u$ we define the following reward functions following \eqref{eq:reward-a-b}:
\[
\tilde{r}_i(s,a) = \frac{\left\langle r_i, q_1\right \rangle}{H}
\ \textrm{and} \ 
\tilde{r}_i(s,b) = \frac{\left\langle r_i, q_2\right \rangle}{H}.
\]
Then we have 
$V^{\pi^a}(\tilde{r}_i) = \left\langle r_i, q_1\right \rangle = V^{\pi_1}(r_i)$ 
and $V^{\pi^b}(\tilde{r}_i) = \left\langle r_i, q_2\right \rangle = V^{\pi_2}(r_i)$. 
By Axiom \ref{axiom:IIA} it must be that $W\left(\pi_a; {\rr}\right) \ge W\left(\pi_b; {\rr} \right)$; hence, we have $u \succcurlyeq v$ by definition. 
This implies that $F(u) \ge F(v)$ as $F$ satisfies the first requirement of Lemma~\ref{lmm:osborne}. 
By Lemma~\ref{lem:characterization-F} this is equivalent to:
\begin{align*}
&\prod_{i \in [n]} u_i \ge \prod_{i\in [n]} v_i \\
\iff& \sum_{i \in [n]} \log(\left\langle r_i, q_1\right \rangle) \ge \sum_{i \in [n]} \log(\left\langle r_i, q_2\right \rangle) \\
\iff& \sum_{i \in [n]} \log\left(V^{\pi_1}(r_i) \right) \ge \sum_{i \in [n]} \log \left(V^{\pi_2}(r_i) \right) \\
\iff& \NW(\pi_1; \rr) \ge \NW(\pi_2; \rr).
\end{align*}
% \jiarui{Is it also necessary to argue $\NW(\pi_1; \rr) \ge \NW(\pi_2; \rr) \Longrightarrow W\left(\pi_1; {\rr}\right) \ge W\left(\pi_2; {\rr} \right)$? Now we only have the other direction.}
It can be also shown that $F(u) \ge F(v)  \implies u \succcurlyeq v \implies W\left(\pi_1; {\rr}\right) \ge W\left(\pi_2; {\rr} \right)$.
Hence, we get that $W\left(\pi_1; {\rr}\right) \ge W\left(\pi_2; {\rr} \right) \iff \NW(\pi_1; \rr) \ge \NW(\pi_2; \rr)$.

\section{Learning Algorithm}
The learning algorithm first relabels the states so that each state appears only once in an episode of length $H$. This can be done by increasing the state space by a factor of $H$. In particular, if state $s$ appears at time step $h$, we will call that state $s_h$. 

\subsection{Optimistic Planning}
We aim to solve the following  optimization problem.
\begin{align}\label{eq:optimistic-planning-gen}
    \tilde{P}_t \in \argmax_{P \in C_t(\widehat{P})} \max_\pi \F(\pi; P)
\end{align}
where $C_t(\widehat{P})$ is the set of plausible transition functions at time $t$ i.e.
$$
C_t(\widehat{P}) = \set{P: \forall s \norm{P(s,a,\cdot) - \widehat{P}(s,a,\cdot)}_1 \le \varepsilon_t(s,a) }
$$
for $\varepsilon_t(s,a) = \sqrt{\frac{4S\log(SAt/\delta)}{\max \set{1, N_t(s,a)} }}$. 

We will consider an \emph{extended MDP} $\tilde{M}$ with state space $S$. For each state $s$ the set of available actions is given as
$$
\calA_s = \set{(a,P_{s,a}) : \norm{P_{s,a} - \widehat{P}(s,a,\cdot)}_1 \le \varepsilon_t(s,a)}
$$
Moreover, taking an action $(a,\tilde{P}_{s,a})$ at state $s$ leads to transitioning to the next state $s'$ with probability $\tilde{P}(s'\mid s,a)$. The reward for agent $i$ is the same as in the original MDP i.e. $r_i(s; a, \tilde{P}_{s,a}) = r_i(s,a)$. First, note that for the three objectives $\F \in \set{\NW, \MW, \GGW}$ the value of the objective in \eqref{eq:optimistic-planning-gen} is the same as the value of the optimal fair policy in the extended MDP i.e.
$$
\max_{P \in C_t(\widehat{P})} \max_\pi \F(\pi; P) = \max_\pi \F(\pi; \tilde{M})
$$
In order to see why this is the case, given any model $\tilde{P} \in C_t(\widehat{P})$ and a policy ${\pi}$ we can consider a policy $\tilde{\pi}$ for the extended MDP such that 
$$
\tilde{\pi}(a, P_{s,a} \mid s) = \left\{
\begin{array}{cc}
   \pi(a \mid s)  & \textrm{ if } P_{s,a} = \tilde{P}(s,a,\cdot) \\
    0 & \textrm{o.w.} 
\end{array}\right.
$$
Notice that from the relabeling of the state space, each state appears at most once in an episode of length $H$ and therefore a state $s$ is always associated with a unique action $a$ and hence the choice of $P_{s,a}$ in the above definition is well-defined. Now from the definition of extended MDP we have $V_i(\tilde{\pi}; \tilde{M}) = V_i(\pi; \tilde{P})$ for each agent $i$. Then for any function $\F \in \set{\NW, \MW, \GGW}$ the objective in the extended MDP under policy $\tilde{\pi}$ is the same as the objective in model $\tilde{P}$ and policy $\pi$. 

On the other hand, suppose we are given a policy $\pi'$ that is optimal for the extended MDP. Without loss of generality, we can assume that $\pi'$ is deterministic and for any state $s$ and action $a$ we can choose the model $P_{s,a}$ so that $\pi'(a,P_{s,a} \mid s) = 1$. Note that there is no ambiguity in the choice of $P_{s,a}$ since each state appears only once in an episode of length $H$. We now show how to compute the optimal policy in the extended MDP.

\subsubsection{Minimum Welfare} Let $q_h(s; a, \tilde{P}_{s,a})$ be the probability that the policy visits state $s$ at time-step $h$ and takes action $(a, \tilde{P}_{s,a})$. The optimal policy can be computed through the following linear program.
\begin{align}
\label{eq:lp-extended-MDP-MW}
    \begin{split}
        \max_{q,t}\ &t\\
        \textrm{s.t.}\ &\sum_h \sum_{s,a} \sum_{\tilde{P}_{s,a}} q_h(s; a, \tilde{P}_{s,a}) r_i(s,a) \ge t\ \forall i\\
        &\sum_a \sum_{\tilde{P}_{s,a}} q_1(s; a, \tilde{P}_{s,a}) = \rho(s) \ \forall s\\
        & \sum_{a, \tilde{P}_{s,a}} q_{h+1}(s; a, \tilde{P}_{s,a}) = \sum_{s',a, \tilde{P}_{s',a}} q_{h}(s'; a, \tilde{P}_{s',a}) \tilde{P}_{s',a}(s)\\
        &\qquad
\ \forall s, h \ge 1\\
        &q_{h}(s; a, \tilde{P}_{s,a}) \ge 0
\end{split}
\end{align}
Note that this linear program has infinitely many variables but only $O(n + HS)$ constraints. Therefore, we consider its dual formulation.

\begin{align}
   % \begin{split}
        \min_{u,V}\ &\sum_s \rho(s) V_1(s) \nonumber\\
        \textrm{s.t.}\ & V_h(s) \ge \sum_i u_i r_i(s,a) + \sum_{s'} \tilde{P}_{s,a}(s') V_{h+1}(s') \nonumber\\
        &\qquad \ \forall s, a, \tilde{P}_{s,a}, h < H\label{eq:constraint-1-dual-maxmin}\\
        &V_H(s) \ge \sum_i u_i r_i(s,a)\ \forall s, a\nonumber\\
        &\sum_i u_i \ge 1\nonumber\\
        &u_i \ge 0 \ \forall i \nonumber
    %\end{split}
\end{align}
This optimization problem has $n+HS$ variables but infinite number of constraints. We now show that it is possible to efficiently solve for a separation oracle so that one can apply the ellipsoid method to solve this optimization problem. Given variables $u,V$, the last three constraints can be checked easily. For checking the last constraint~\cref{eq:constraint-1-dual-maxmin} consider the following optimization problem for each $s,a$ and $h$.
\begin{align*}
    \max_{\widetilde{P}_{s,a}}&\ \sum_i u_i r_i(s,a) + \sum_{s'}\widetilde{P}_{s,a}(s') V_{h+1}(s')\\
    \textrm{s.t}&\ \norm{\widetilde{P}_{s,a} - \widehat{P}(s,a,\cdot)}_1 \le \varepsilon_t(s,a)
\end{align*}
This optimization problem is a concave optimization problem over a convex set. If the objective value is larger than $V_h(s)$ then the corresponding solution (i.e. $\widetilde{P}_{s,a}$) provides with a violating constraint. If the objective is less than $V_h(s)$ for all $s,h,a$ then the variables $u, V$ are feasible.

\subsubsection{Generalized Gini Welfare} 
Let $q_h(s; a, \tilde{P}_{s,a})$ be the probability that the policy visits state $s$ at time-step $h$ and takes action $(a, \tilde{P}_{s,a})$. The optimal policy can be computed through the following linear program.
\begin{align}
\begin{split}
\max_{q,t}\ &t\\
\textrm{s.t.} &\sum_i w_i \sum_{h=1}^H \sum_{s,a} \sum_{\tilde{P}_{s,a}} q_h(s; a, \tilde{P}_{s,a})  r_{\sigma(i)}(s,a) \ge t \ \forall \sigma \in \calS_n\\
 &\sum_a \sum_{\tilde{P}_{s,a}} q_1(s; a, \tilde{P}_{s,a}) = \rho(s) \ \forall s\\
        & \sum_{a, \tilde{P}_{s,a}} q_{h+1}(s; a, \tilde{P}_{s,a}) = \sum_{s',a, \tilde{P}_{s',a}} q_{h}(s'; a, \tilde{P}_{s',a}) \tilde{P}_{s',a}(s)\\
        &\qquad
\ \forall s, h \ge 1\\
        &q_{h}(s; a, \tilde{P}_{s,a}) \ge 0
\end{split}
\end{align}
We now construct the dual  problem of the optimization problem shown above.
\begin{align}
   % \begin{split}
        \min_{u,V}\ &\sum_s \rho(s) V_1(s) \nonumber\\
        \textrm{s.t.}\ & V_h(s) \ge \sum_i w_i \sum_\sigma u_\sigma  r_{\sigma(i)}(s,a) + \sum_{s'} \tilde{P}_{s,a}(s') V_{h+1}(s') \nonumber\\
        &\qquad \ \forall s, a, \tilde{P}_{s,a}, h < H\nonumber\\
        &V_H(s) \ge \sum_i w_i \sum_\sigma u_\sigma  r_{\sigma(i)}(s,a)\ \forall s, a\nonumber\\
        &\sum_\sigma u_\sigma \ge 1\nonumber\\
        &u_\sigma \ge 0 \ \forall \sigma \nonumber
    %\end{split}
\end{align}
This linear program has $n! + HS$ variables. We now show that a separation oracle of this problem can be constructed in time $\textrm{poly}(S,A,H) \cdot n! \cdot n!$ time. The last three constraints can be checked in time $SA \cdot n!$ time. In order to check the first constraint consider the following optimization problem for each $s,a$ and $h$.
\begin{align*}
    \max_{\widetilde{P}_{s,a}} &\sum_i w_i \sum_\sigma u_\sigma r_{\sigma(i)} (s,a) + \sum_{s'} \widetilde{P}_{s,a}(s') V_{h+1}(s')\\
    \textrm{s.t.} \ &\norm{\widetilde{P}_{s,a} - \widehat{P}(s,a,\cdot)}_1 \le \varepsilon_t(s,a)
\end{align*}
Note that this optimization is a concave optimization problem but evaluating the objective for any choice of the parameter $\widetilde{P}_{s,a}$ takes time $SA \cdot n \cdot n!$ time. This implies that the optimization problem above can be solved in time $\textrm{poly}(S,A) \cdot n \cdot n!$. Since we need to solve the above optimization problem for each $s,a$, and $h$ the total running time for the separation oracle is $\textrm{poly}(S,A,H) \cdot n \cdot n!$.

For a linear programming with $d$ variables, ellipsoid method runs in $\tilde{O}(d^2)$ iterations. Substituting $d = n! + HS$ and running time of the sepration oracle, we get that the total running time to solve the optimistic optimization problem for the objective $\GGW$ is $\textrm{poly}(S,A,H) \cdot n \cdot (n!)^3 = n^{O(n)} \cdot \textrm{poly}(S,A,H)$.

\subsubsection{Nash Welfare}
The primal optimization problem is given as follows.
\begin{align}
\label{eq:lp-extended-MDP-NW}
    \begin{split}
        \max_{q}\ &\sum_{i=1}^n \log \left( \sum_{h=1}^H \sum_{s,a} \sum_{\widetilde{P}_{s,a}} q_h(s; a, \widetilde{P}_{s,a}) r_i(s,a)\right)\\
        \textrm{s.t.}\ &\sum_a \sum_{\tilde{P}_{s,a}} q_1(s; a, \tilde{P}_{s,a}) = \rho(s) \ \forall s\\
        & \sum_{a, \tilde{P}_{s,a}} q_{h+1}(s; a, \tilde{P}_{s,a}) = \sum_{s',a, \tilde{P}_{s',a}} q_{h}(s'; a, \tilde{P}_{s',a}) \tilde{P}_{s',a}(s)\\
        &\qquad
\ \forall s, h \ge 1\\
        &q_{h}(s; a, \tilde{P}_{s,a}) \ge 0
\end{split}
\end{align}
As before, this optimization problem has infinite number of variables but $HS$ constraints. Therefore, we look at the corresponding dual problem.
\begin{align*}
   % \begin{split}
        \min_{\beta,V}\ &-\sum_{i=1}^n \log(\beta_i) + \sum_s \rho(s) V_1(s) \nonumber\\
        \textrm{s.t.}\ & V_h(s) \ge \sum_i { r_i(s,a) }{\beta_i} + \sum_{s'} \tilde{P}_{s,a}(s') V_{h+1}(s') \\
        &\qquad \ \forall s, a, \tilde{P}_{s,a}, h < H\\
        &V_H(s) \ge \sum_i { r_i(s,a)}{\beta_i}\ \forall s, a\\
        &\sum_i \beta_i \ge 1 \\
        &\beta_i \ge 0\ \forall i 
    %\end{split}
\end{align*}
Now we can follow the same steps as the objective $\MW$ and construct an efficient separation oracle for the dual problem shown above. Finally observe that the dual problem has $O(n+HS)$ variables and can be solved efficiently through ellipsoid method.

\subsection{Proof of Theorem~\ref{thm:regret-nsw}}
\begin{proof}
We will write $\calS = S\cdot H$.
Let $P^\star$ be the true probability transition function.  Let us define $\varepsilon_t(s,a) = \sqrt{\frac{4 S \log(SAt/\delta)}{1 \vee N_t(s,a)}}$. By the Chernoff-Hoeffding type concentration inequality for categorical random variables, the following bound holds for any state $s$, action $a$, and episode $t$.
$$
\Pr\left( \norm{\hat{P}(s,a,\cdot) - P^\star(s,a,\cdot)}_1 \ge \varepsilon_t(s,a) \right) \le \frac{\delta}{2 t^2 SA}
$$
Therefore, by a union bound over the $T$ episodes and all state-action pairs we see that with high probability the true probability transition function $P^\star$ is contained in the set $C_t(\hat{P})$ for all $t$ i.e.
$$
\Pr \left( \exists t\ P^\star \notin C_t(\hat{P})\right) \le \delta \sum_t \frac{1}{2t^2} \le \delta
$$

So we assume that this event holds. At the start of episode $t$, we compute the best policy $\tilde{\pi}_t$ for the optimistic model $\tilde{P}_t$. This implies the following bound on the Nash welfare with respect to the true model $P^\star$.
\begin{align*}
    \NW(\tilde{\pi}_t; \tilde{P}_t) 
    &= \max_\pi \NW(\pi; \tilde{P}_t) \\
    &\ge \max_{P \in C_t(\widehat{P})} \max_\pi \NW(\pi; P) \ge \max_\pi \NW(\pi; P^\star)
\end{align*}
This allows us to upper bound the regret through the optimistic policy.
\begin{align*}
    \reg_\NW(T) &= \sum_{t=1}^T \NW(\pi^\star_\NW) - \NW(\pi_t) \\
    &\le \sum_{t=1}^{ T}  \NW(\tilde{\pi}_t; \tilde{P}_t) - \NW(\tilde{\pi}_t; P^\star)\\
    &\le H^{n-1}  \sum_{t=1}^{T}\sum_{i \in [n]} \abs{V^{\tilde{\pi}_t}\left(r_i; \tilde{P}_t\right) - V^{\tilde{\pi}_t}\left(r_i; P^\star\right) }\\
    &\le H^{n-1} \sum_{i=1}^n  \underbrace{\sum_{t = 1}^T  \abs{V^{\tilde{\pi}_t}\left(r_i; \tilde{P}_t\right) - V^{\tilde{\pi}_t}\left(r_i; P^\star\right) }}_{:= \reg_i}
\end{align*}
The second inequality uses lemma \ref{lem:nsw-to-value}. We now use lemma \ref{lem:diff-in-value} to establish a bound of $\tilde{O}\left(H^2S\sqrt{AT}\right)$ on the term $\reg_i$ for any $i$. This proves the desired upper bound on the regret.

Consider any agent $i$. We can apply lemma \ref{lem:diff-in-value} with $ \epsilon_t(s',a,s) = \abs{\tilde{P}_t(s',a,s) - P^\star(s',a,s)}$. Notice that as $\tilde{P}_t, P^\star \in C_t(\hat{P})$ we have the following inequality.

\begin{align*}
\sum_s \epsilon_t(s',a,s) &= \sum_s \abs{\tilde{P}_t(s',a,s) - P^\star(s',a,s)} \\
&\le \sum_s \abs{\hat{P}_t(s',a,s) - P^\star(s',a,s)} \\
&+ \sum_s \abs{\tilde{P}_t(s',a,s) - P^\star(s',a,s)}\\
&\le 2 \sqrt{\frac{4S \log(SAt/\delta)}{1 \vee N_{t-1}(s',a)}}
\end{align*}
We will write $q_{t,h}(s,a)$ to denote the probability that the policy $\tpi$ visits state $s$ at time-step $h$ and takes action $a$.
\begin{align*}
\reg_i &= \sum_{t = 1}^T  \abs{V^{\tilde{\pi}_t}\left(r_i; \tilde{P}_t\right) - V^{\tilde{\pi}_t}\left(r_i; P^\star\right) } \\&\le H \sum_{t=1}^T \sum_{h=1}^{H-1} \sum_{s',b} q_{t,h}(s',b) \epsilon_t(s',b)\\
&\le 2 H \sum_{t=1}^T \sum_{h=1}^{H-1} \sum_{s',b} q_{t,h}(s',b) \sqrt{\frac{4S \log(SAt/\delta)}{1 \vee N_{t-1}(s',b)} }\\
%&= 4H^2 \sqrt{S \log (SAT/\delta)} \sum_t \sqrt{\sum_{s',b} \frac{1}{N_t(s',b)}}\\
&\le 4 H \sqrt{S \log (SAT /\delta)} \sum_{s,a} \sum_{t=1}^T \sum_{h=1}^{H-1} \frac{q_{t,h}(s,a) }{\sqrt{1 \vee N_t(s,a)}} \\
&= 4 H \sqrt{S \log (SAT /\delta)} \sum_{s,a} \sum_{t=1}^T \frac{\E[n_{t}(s,a)] }{\sqrt{1 \vee N_{t-1}(s,a)}} \\
%&\le 4H^2 \sqrt{S \log (SAT/\delta)} \sqrt{T}\sqrt{\sum_t \sum_{s',b} \frac{1}{N_t(s',b) }}\\
%&= 4H^2 \sqrt{S T \log (SAT/\delta)} \sqrt{\sum_{s',b} \sum_{t=1}^{N_T(s',b)} \frac{1}{t} }\\
%&\le 4H^2 \sqrt{S \log (SAT/\delta)} \sqrt{SA \log T}
%&\qedhere
\end{align*}
where we used $n_t(s,a)$ denotes the number  of times the policy $\tpi$ visits state $s$ and takes action $a$ in episode $t$. Note that $\E[n_t(s,a)] = \sum_{h=1}^{H-1} q_{t,h}(s,a)$.
In order to bound the final term, we first bound the term $\E[n_t(s,a)] - n_t(s,a)$ for any state, action $s,a$ and time step $t$. Let $X_h = \sum_{j \le h}\left(\one \set{s_{t,h}=s, a_{t,h}=a} - q_{t,h}(s,a)\right)$. Then $\{X_h\}_{h \ge 1}$ is a Martingale. Moreover, $\abs{X_h - X_{h-1}} \le 1$. Therefore, by the Azuma-Hoeffding inequality we have,
    $$
    \Pr\left( \abs{X_H} \ge \sqrt{2H \log(SA T / \delta)}\right) \le \frac{\delta}{SAT}
    $$
    Therefore, by a union bound we have that $\abs{n_t(s,a) - \E[n_t(s,a)]} \le \sqrt{2H \log(SA T / \delta)}$ for any state, action $s,a$ and episode $t$ with probability at least $1-\delta$. Therefore, we can assume that this event holds. Conditioned on this event we have, 
\begin{align*}
        &\sum_{s,a} \sum_{t=1}^T \frac{\E[n_{t}(s,a)] }{\sqrt{1 \vee N_{t-1}(s,a)}} \\
        &\le \sqrt{2H \log(SAT/\delta)} \sum_{s,a} \sum_{t=1}^T \frac{n_{t}(s,a)  }{\sqrt{1 \vee N_{t-1}(s,a)}} \\
        &\le 3 \sqrt{2H \log(SAT/\delta)}\sqrt{SA}\sqrt{\sum_{s,a}N_T(s,a)}\\
        &\le 3 \sqrt{2H \log(SAT/\delta)}\sqrt{SA} \sqrt{HT}\\
        &\le 5 H \sqrt{SAT \log(SAT/\delta)}
     \end{align*}
    The second inequality follows from a standard result relating to a sequence of non-negative integers (see 38.22 from \cite{LS20}). Substituting the last bound in the upper bound on $\reg_i$ we get the following upper bound on $\reg_i$.
    $$
    \reg_i \le 20 H^2 S \sqrt{A T} \log(SAT/\delta) \qedhere
    $$
\end{proof}

% \lmmnswtovalue*

\begin{lemma}\label{lem:nsw-to-value}
For any policy $\pi$, and probability transition functions $P_1$ and $P_2$ we have
\begin{align*}
&\NW(\pi; P_1) - \NW(\pi; P_2) \\ \le &H^{n-1} \sum_{i=1}^n \abs{V^\pi(r_i; P_1) - V^\pi(r_i; P_2)}
\end{align*}
\end{lemma}

\begin{proof}
We use induction on the number of agents $n$. When $n=1$ this is trivially true. So we assume the claim holds for $n=m$. 
\begin{align*}
    &\NW(\pi; P_1) - \NW(\pi; P_2) \\
    = &\prod_{i=1}^{m+1}V^\pi(r_i; P_1) - \prod_{i=1}^{m+1} V^\pi(r_i; P_2)\\
    &= \prod_{i=1}^m V^\pi(r_i; P_1) \cdot V^\pi(r_{m+1}; P_1) \\- &\prod_{i=1}^m V^\pi(r_i; P_1) \cdot V^\pi(r_{m+1}; P_2) \\
    &+ \prod_{i=1}^m V^\pi(r_i; P_1) \cdot V^\pi(r_{m+1}; P_2) \\ -& \prod_{i=1}^m V^\pi(r_i; P_2) \cdot V^\pi(r_{m+1}; P_2)\\
    &\le \prod_{i=1}^m V^\pi(r_i; P_1) \abs{V^\pi(r_{m+1}; P_1) - V^\pi(r_{m+1}; P_2)} \\
    + &V^\pi(r_{m+1}; P_2) \left(\prod_{i=1}^m V^\pi(r_i; P_1) - \prod_{i=1}^m V^\pi(r_i; P_2) \right)\\
    &\le H^m \abs{V^\pi(r_{m+1}; P_1) - V^\pi(r_{m+1}; P_2)}\\
    + &H \cdot H^{m-1} \sum_{i=1}^m \abs{V^\pi(r_{i}; P_1) - V^\pi(r_{i}; P_2)}\\
    &\textrm{[By the induction hypothesis }\\
    &\ \textrm{ and the fact that $V^\pi(r_i; P) \le H$ for any $i$ and $P$]}\\
    &\le H^m \sum_{i=1}^{m+1} \abs{V^\pi(r_{i}; P_1) - V^\pi(r_{i}; P_2)} 
    &\qedhere
\end{align*}
\end{proof}

\begin{lemma}\label{lem:diff-in-value}
Let $\epsilon(s',a) = \norm{P(s',a,\cdot) - \tilde{P}(s',a,\cdot)}_1$ for all tuples $(s',a)$. Then for any policy $\tilde{\pi}$ we have
$$
\abs{V^{\tpi}(r;\tilde{P}) - V^{\tpi}(r;P)} \le H \sum_{h=1}^{H-1} \sum_{s',b} q_{h}(s',b)    \epsilon(s',b)
$$
where $q_h(s,a)$ is the probability that policy $\tpi$ visits state $s$ at time-step $h$ and takes action $a$.
\end{lemma}
\begin{proof}
We will use $\tilde{q}$ to write the state-action occupancy measure under policy $\tilde{\pi}$ and probability transition function $\tilde{P}$ i.e. $\tilde{q}_h(s,a) = \Pr_{\tilde{\pi},\tilde{P}}(s_h = s, a_h = a)$. Similarly, we will use $q$ to denote the state-action occupancy measure under policy $\tilde{\pi}$ and probability transition function $P$. We will also write $\epsilon(s',b,s) = \tilde{P}(s',b,s) - P(s',b,s)$.

Since rewards are bounded between $0$ and $1$ we have the following inequality.
\begin{align}
&\abs{V^{\tpi}(r; \tilde{P}) - V^{\tpi}(r; P)} \\
=& \abs{\sum_{h=1}^H \sum_{s,a} \tilde{q}_h(s,a) r(s,a) - \sum_{h=1}^H \sum_{s,a} q_h(s,a) r(s,a)}\nonumber \\
&\le \sum_{h=1}^H \underbrace{\sum_{s,a} \abs{\tilde{q}_h(s,a) - q_h(s,a)}}_{:=\Delta_h} \label{eq:temp-value-difference}
\end{align}
We now establish a recurrence relation for the term $\Delta_h = \sum_{s,a} \abs{\tilde{q}_h(s,a) - q_h(s,a)}$. First note that, $\Delta_1 = 0$ since
\begin{align*}
\Delta_1 &= \sum_{s,a} \abs{\tilde{q}_1(s,a) - q_1(s,a)} \\
&= \sum_{s,a} \abs{\rho(s) \tpi(s,a) - \rho(s) \tpi(s,a)} = 0
\end{align*}
Here we use the equality constraint $\sum_{a} q_1(s,a) = \rho(s) = \sum_a \tilde{q}_1(s,a)$.
\begin{align*}
&\Delta_h = \sum_{s,a} \abs{\tilde{q}_h(s,a) - q_h(s,a)} \\
&= \sum_{s,a} \abs{\sum_{b} \tilde{q}_h(s,b) \tpi(a|s) - \sum_b q_h(s,b) \tpi(a|s)}\\
&\le \sum_{s,a} \tpi(a|s) \abs{\sum_{b} \tilde{q}_h(s,b)  - \sum_b q_h(s,b) }\\
&\le \sum_s \abs{\sum_{s',b} \tilde{q}_{h-1}(s',b) P(s',b,s) - \sum_{s',b} q_{h-1}(s',b) \tilde{P}(s',b,s)}\\
&\le \sum_s \abs{\sum_{s',b} \left(\tilde{q}_{h-1}(s',b) - q_{h-1}(s',b) \right) P(s',b,s)}
\\&+ \sum_s \abs{\sum_{s',b} q_{h-1}(s',b) \left(P(s',b,s) - \tilde{P}(s',b,s) \right)}\\
\end{align*}
\begin{align*}
&\le \sum_s \sum_{s',b} \abs{\tilde{q}_{h-1}(s',b) - q_{h-1}(s',b)} P(s',b,s) \\
&+ \sum_s \sum_{s',b} q_{h-1}(s',b) \epsilon(s',b,s)\\
&\le \sum_{s',b} \abs{\tilde{q}_{h-1}(s',b) - q_{h-1}(s',b) }
\\&+ \sum_{s',b} q_{h-1}(s',b) \sum_s   \epsilon(s',b,s)\\
&\le \Delta_{h-1} + \sum_{s',b} q_{h-1}(s',b) \sum_s   \epsilon(s',b,s)
%&\le \Delta_{h-1} + \sqrt{\sum_{s',b} q_{h-1}(s',b)^2} \sqrt{\sum_{s',b}\left( \sum_s \epsilon(s',b,s) \right)^2}\\
%&\le \Delta_{h-1} + \sum_{s',b} q_{h-1}(s',b) \sqrt{\sum_{s',b} \norm{\epsilon(s',b,\cdot)}_1^2 }\\
%&\le \Delta_{h-1} + \underbrace{\sqrt{\sum_{s',b} \norm{\epsilon(s',b,\cdot)}_1^2 }}_{:=\varepsilon}
\end{align*}
The above recurrence relation and $\Delta_1 = 0$ gives us the following bound on $\Delta_h$.
$$
\Delta_h \le \sum_{h'=1}^{h-1} \sum_{s',b} q_{h'}(s',b) \sum_s   \epsilon(s',b,s)
$$
%$\Delta_h = (h-1)\varepsilon$. Substituting this result in equation~\ref{eq:temp-value-difference} we get the following bound.
%\[
%\abs{V^{\tpi}(r; \tilde{P}) - V^{\tpi}(r; P)} \le \sum_{h=1}^H (h-1)\varepsilon \le H^2 \varepsilon
% \qedhere
%\]
Substituting the above bound in \cref{eq:temp-value-difference} we get the following bound.
\begin{align*}
&\abs{V^{\tpi}(r; \widetilde{P}) - V^{\tpi}(r; P)} \\
&\le \sum_{h=1}^H \sum_{h'=1}^{h-1} \sum_{s',b} q_{h'}(s',b) \sum_s   \epsilon(s',b,s)\\
&\le H \sum_{h=1}^{H-1} \sum_{s',b} q_{h}(s',b)    \norm{\epsilon(s',b,\cdot)}_1
\end{align*}
\end{proof}

\subsection{Proof of Theorem~\ref{thm:regret-ggw}}
\begin{proof}
We proceed very similarly to the proof of \Cref{thm:regret-nsw}. Since the true probability transition function, $P^\star$ is contained within $C_t(\widehat{P})$ with high probability and algorithm~\ref{algo:ucrl-fair} uses optimistic planning, we can bound the regret as follows.
\begin{align*}
    \reg_\GGW(T) \le \sum_{t=1}^T \GGW(\tilde{\pi}_t; \tilde{P}_t) - \GGW(\tilde{\pi}_t; P^\star)
\end{align*}
We now consider two ordering of the agents. Let $i_1,i_2,\ldots,i_n$ be an ordering of the agents so that
$$
V^{\tilde{\pi}_t}(r_{i_1}; \tilde{P}_t) \ge V^{\tilde{\pi}_t}(r_{i_2}; \tilde{P}_t) \ge \ldots \ge V^{\tilde{\pi}_t}(r_{i_n}; \tilde{P}_t).
$$
Let $\ell_1,\ell_2,\ldots,\ell_n$ be an ordering of the agents so that
$$
V^{\tilde{\pi}_t}(r_{\ell_1}; P^\star) \ge V^{\tilde{\pi}_t}(r_{\ell_2}; P^\star) \ge \ldots \ge V^{\tilde{\pi}_t}(r_{\ell_n}; P^\star).
$$
Without loss of generality, we can assume that $i_k = k$ for all $k \in [n]$.
Then we can rewrite the upper bound on regret as follows:
\begin{align*}
    &\reg_\GGW(T) \le \sum_{t=1}^T \sum_{k=1}^n w_k \left(V^{\tilde{\pi}_t}(r_{k}; \tilde{P}_t) - V^{\tilde{\pi}_t}(r_{\ell_k}; P^\star) \right).
\end{align*}
We can show that for all $k \in [n]$ it holds that
\begin{equation}
\label{eq:GGW-V-P}
V^{\tilde{\pi}_t}(r_{k}; \tilde{P}_t) - V^{\tilde{\pi}_t}(r_{\ell_k}; P^\star)
\le 
\max_{i\in[n]} |V^{\tilde{\pi}_t}(r_{i}; \tilde{P}_t) - V^{\tilde{\pi}_t}(r_{i}; P^\star)|.
\end{equation}
Indeed, if $V^{\tilde{\pi}_t}(r_{k}; \tilde{P}_t) \le V^{\tilde{\pi}_t}(r_{\ell_k}; P^\star)$, then we are done, so we assume that $V^{\tilde{\pi}_t}(r_{k}; \tilde{P}_t) > V^{\tilde{\pi}_t}(r_{\ell_k}; P^\star)$.
 
Note that there must be $j \ge k$ such that $\ell_j \le k$.
% , where $\ell^{-1}$ maps each $j$ to an index $x \in [n]$ such that $\ell_x = j$.
This is because the converse, i.e., $\ell_j > k$ for all $j \ge k$, would imply that $\ell_j = \ell_{j'}$ for some $j \neq j'$, which is a contradiction.
It then follows that
\begin{align*}
&V^{\tilde{\pi}_t}(r_{k}; \tilde{P}_t) - V^{\tilde{\pi}_t}(r_{\ell_k}; P^\star)  \\
&\le 
V^{\tilde{\pi}_t}(r_{\ell_j}; \tilde{P}_t) - V^{\tilde{\pi}_t}(r_{\ell_j}; P^\star) \\
&= |V^{\tilde{\pi}_t}(r_{\ell_j}; \tilde{P}_t) - V^{\tilde{\pi}_t}(r_{\ell_j}; P^\star)|,
\end{align*}
so \eqref{eq:GGW-V-P} follows.

% \begin{align*}
%     &\reg_\GGW(T) \le \sum_{t=1}^T \sum_{k=1}^n w_k \left(V^{\tilde{\pi}_t}(r_{i_k}; \tilde{P}_t) - V^{\tilde{\pi}_t}(r_{\ell_k}; P^\star) \right)\\
%     &\le \sum_{t=1}^T \sum_{k=1}^n w_k \abs{V^{\tilde{\pi}_t}(r_{i_k}; \tilde{P}_t) - V^{\tilde{\pi}_t}(r_{i_k}; P^\star)}\\
%     &+ \sum_{t=1}^T \sum_{k=1}^n w_k \left( V^{\tilde{\pi}_t}(r_{i_k}; P^\star) -  V^{\tilde{\pi}_t}(r_{\ell_k}; P^\star) \right)\\
%     &\le \sum_{t=1}^T \sum_{k=1}^n w_k \abs{V^{\tilde{\pi}_t}(r_{i_k}; \tilde{P}_t) - V^{\tilde{\pi}_t}(r_{\ell_k}; P^\star)}
% \end{align*}
% The last inequality follows from the observation that the ordering $\{\ell_k\}_{k \in [n]}$ is the optimal ordering for generalized Gini welfare i.e. minimizes the weighted sum of values. Any other ordering, and in particular $\{i_k\}_{k \in [n]}$ can only increase the weighted sum of values. Now rearranging the summation we get the following upper bound.
% \begin{align*}
%     \reg_\GGW(T) \le \sum_{k=1}^n w_k \underbrace{\sum_{t=1}^T \abs{V^{\tilde{\pi}_t}(r_{i_k}; \tilde{P}_t) - V^{\tilde{\pi}_t}(r_{\ell_k}; P^\star)} }_{:=\reg_i}
% \end{align*}
% We now substitute an upper bound of $\tilde{O}(H^2 S \sqrt{AT})$ on the term $\reg_i$ and use the assumption $\sum_{k} w_k = 1$ to complete the proof.

Hence, we have 
\begin{align*}
    \reg_\MW(T) \le \sum_{t=1}^T \max_i \abs{V^{\tpi_t}(r_i; P^\star) - V^{\tpi_t}(r_i; \tilde{P}_t)}
\end{align*}
and the remainder of the proof is the same as that of \Cref{thm:regret-mm}.
\end{proof}

\subsection{Unknown Reward Functions}
When the reward function is unknown we update the optimistic planning step as follows
\begin{align*}
    &(\tilde{P}_t, \tilde{r}_t) \leftarrow \argmax_{P \in C_t(\widehat{P}), r \in D_t(\hat{r})} \max_\pi \F(\pi; P, r)\\
    &\tilde{\pi}_t \leftarrow \argmax_\pi \F(\pi; \widetilde{P}_t, \tilde{r}_t)
\end{align*}
Here $D_t(\hat{r})$ is a confidence set around the empirical reward function $\hat{r}$ and it can be constructed by using Chernoff-Hoeffding inequality and the union bound. Here we consider the case $\F = \NW$. The proof for $\MW$, and $\GGW$ are similar. As in the proof of theorem~\ref{thm:regret-nsw} we can upper bound regret as
\begin{align*}
    &\reg_\NW(T) \le H^{n-1}  \sum_{t=1}^{T}\sum_{i \in [n]} \abs{V^{\tilde{\pi}_t}\left(\tilde{r}_{t,i}; \tilde{P}_t\right) - V^{\tilde{\pi}_t}\left(r_i; P^\star\right) }\\
    &\le H^{n-1} \sum_{i \in [n]} \underbrace{\sum_{t=1}^T  \abs{V^{\tilde{\pi}_t}\left(\tilde{r}_{t,i}; \tilde{P}_t\right) - V^{\tilde{\pi}_t}(r_i; \tilde{P}_t)}}_{:= \reg_i^r}\\
    &+ H^{n-1}  \sum_{i \in [n]} \underbrace{\sum_{t=1}^T \abs{V^{\tilde{\pi}_t}\left(r_i; \tilde{P}_t\right) - V^{\tilde{\pi}_t}(r_i; P^\star)}}_{:= \reg^P_i}
\end{align*}
The second term $\reg^P_i$ exactly equals the term $\reg_i$ introduced in theorem~\ref{thm:regret-nsw} and is bounded by $O(H^2 S \sqrt{AT})$. The first term $\reg^r_i$ equals the difference in value functions between optimistic reward function $\tilde{r}_{t,i}$ and true reward function $r_i$ but with respect to the fixed probability transition function $\tilde{P}_t$. Let $\varepsilon_r(s,a) = \abs{\tilde{r}_{t,i}(s,a) - r_i(s,a)}$ and $\tilde{q}_t$ be the state-action occupancy measure under policy $\tilde{\pi}_t$ and probability transition function $\widetilde{P}_t$. Then we have,
\begin{align*}
    \reg^r_i &= \abs{\left \langle \tilde{q}_t, \tilde{r}_{t,i} \right \rangle - \left \langle \tilde{q}_t, {r}_{i} \right \rangle }= \left \langle \tilde{q}_t, \tilde{r}_{t,i}  - r_i\right \rangle\\
    &\le \sum_h \sum_{s,a} \varepsilon_r(s,a) \tilde{q}_{t,h}(s,a)\\
    &\le H \norm{\varepsilon_r}_1
\end{align*}
Now, by an analysis very similar to bounding the error term $\varepsilon = \sqrt{\sum_{s',b} \norm{\epsilon(s',b,\cdot)}_1^2}$, one can show that the term $\norm{\varepsilon_r}_1$ is bounded by $O(S \sqrt{AT})$. This implies that the term $\reg^P_i$ dominates the term $\reg^r_i$ and the regret is still bounded by $O(nH^{n+1} S \sqrt{AT})$.

\section{Lower Bound on $\reg_\NW(T)$}
\begin{figure}[!h]
\centering
\tikzset{every picture/.style={line width=0.85pt}} %set default line width to 0.75pt        
\resizebox{\columnwidth}{!}{
\input{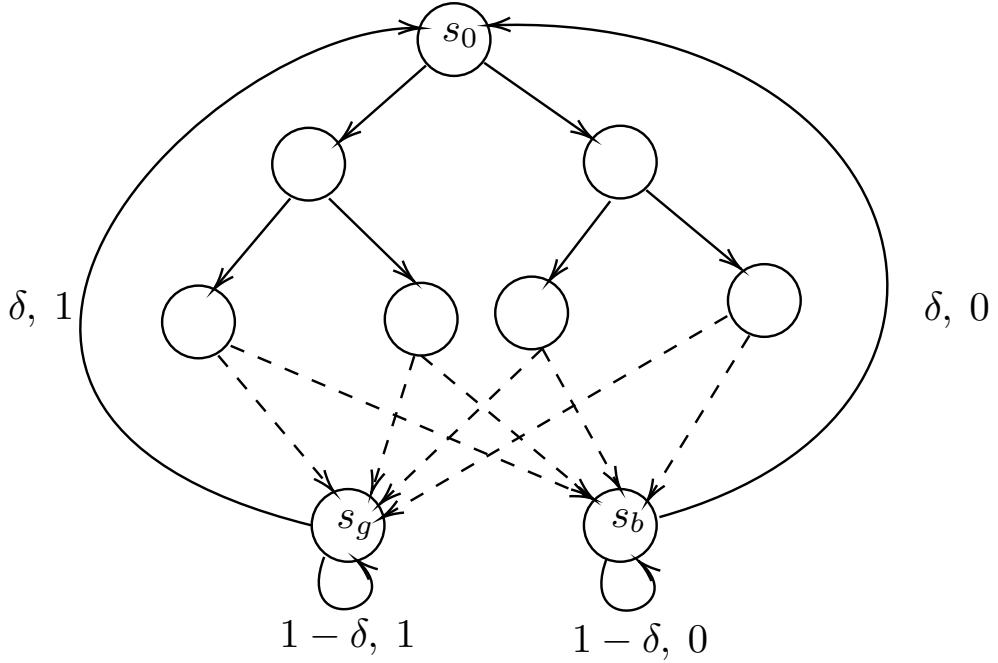}
}
\caption{Lower Bound Instance (following ~\cite{LS20})}
\end{figure}

We consider a  collection of MDPs, where the state-space of each of the MDPs consists of a good state ($s_g$) and a bad state ($s_b$). The remaining $S-2$ states are arranged in a $A$-ary tree. The transitions within the $S$-ary tree is deterministic. Let $\calL$ be the set of leaves in the $A$-ary tree. Then for the $\ell$-th leaf node, and $a$-th action we consider the following transition probabilities.
 $$
 P(s_g|\ell, a) = \frac{1}{2} + \Delta,\quad P(s_b | \ell,a) = \frac{1}{2} - \Delta
 $$
 In the $(\ell,a)$-th MDP, all the other actions from the leaf nodes have uniform probability of transitioning to either the good or the bad state. From the good or the bad state, taking any action maintains the current state with probability $1-\delta$, and transitions to $s_0$ with probability $\delta$. 
 The rewards are $1$ for all the $n$ agents at state $s_g$. All the other rewards are zero. We choose $\delta = \frac{1}{2H}$
 %Note that, once the learner visits state $s_0$ it is faced with a choice of $SA$ transitions (i.e. bandits) and exactly one of them has non-zero rewards e.g. landing in state $s_g$. Moreover, if we choose $\delta = 2/H$ then the learner stays at the good state for $\Theta(H)$ steps and achieves Nash Welfare of $H^n$. Therefore, in expectation, the learner makes $T$ choices over $SA$ bandits with a maximal reward of $H^n$. This gives a lower bound of $\Omega\left( H^n \sqrt{SAT}\right)$.
 
 %Note that the above argument works only in expectation. In particular, the visits to $s_0$ is random and not once per episode. 
 %Moreover, unlike the non-episodic setting, the current state is reset every $H$ steps. Additionally, we need to lower bound the regret in Nash welfare instead of the regret in total reward. For this reason, we provide a full proof of the lower bound in the appendix. 
\subsection{Proof of theorem~\ref{thm:regret-nsw-lower-bound}}
\begin{proof}
We will write $M_0$ to denote the MDP where $P(s_g| \ell, a) = P(s_b| \ell, a) = \frac{1}{2}$ for all $\ell$ and $a$. 
Let us define the following stopping time $\tau$.
\begin{equation}
    \tau = T \wedge \min \set{t: \sum_{u=1}^t\sum_{h=1}^H \one\set{s_{u,h} = s_g} \ge T-1}
\end{equation}
i.e. $\tau$ is the minimum episode number when the number of visits to state $s_g$ is at least $T-1$.  If this episode number is larger than $T$ then we just set $\tau$ as $T$.

Let $T_{\ell,a}$ be the total number of times the policy visits state $\ell$ and takes action $a$ until the stopping time $\tau$. We first show that $\E_0\left[\sum_{\ell,a} T_{\ell, a} \right] = \Theta(T)$ where $\E_0[\cdot]$ denotes expectation with respect to the MDP $M_0$. Note that a visit to one of the leaf nodes is followed by either a visit to the good node $s_g$ or a visit to the bad node $s_b$. Since the number of visits to $s_g$ is $T-1$ we just need to bound the number of visits to the state $s_b$.

Let $T_b$ the total number of visits to the bad state starting from the node $s_0$. We write $T_b = T_b^1 + T_b^2$ where $T_b^1$ is the number of visits to state $s_b$ which were followed by episode reset before visiting the starting state $s_0$. On the other hand, $T_b^2$ is the number of visits to state $s_b$ that were not stopped by episode reset before visiting state $s_0$. Since there are $T$ episodes, we have $T_b^1 \le T$. 

In order to bound $T^2_b$ we consider the total amount of time the policy stays at state $s_b$ before visiting $s_0$. This is a geometric random variable with parameter $1/2H$. Therefore, if we write $Y_b$ to denote the total amount of time the policy stays at $s_b$ because of visitations in the set $T^2_b$, $Y_b$ is a sum of i.i.d. geometric random variables, and by standard concentration inequality~\cite{DP09, brown2011} we get
$$
\Pr\left(Y_b < \frac{1}{2}\cdot T^2_b \cdot \frac{H}{2} \right) \le e^{ -  T^2_b / 6 }
$$
Therefore, as long as $T^2_b \ge 6\log(1/\delta)$ the total time spent at state $s_b$ is at least $T^2_b H/4$ with probability at least $1-\delta$. Since the total number of time steps is exactly $TH$, it must be that $T^2_b \le 4T$ in this case. Therefore, either $T_b^2$ is less than $6\log(1/\delta)$ or it is less than $4T$ with probability at least $1-\delta$. Combining these two cases we get the following upper bound on expected value of $T_b^2$
$$
\E\left[ T_b^2\right] \le 6\log(1/\delta) + (1-\delta) 4T + \delta TH
$$
This bound holds for any choice of $\delta$. In particular, for $\delta = 1/4H$ and $T \ge \Omega(\log H)$ we get that $\E\left[ T_b^2\right] \le 5T$.

Therefore, we have $\E[T_b] = \E[T^1_b] + \E[T^2_b] \le 6T$. As $\E_0\left[ \sum_{\ell, a} T_{\ell, a}\right]$ is bounded by the total number of visits to node $s_g$  and node $s_b$ we have the following bound.
$$
T-1 \le \E_0\left[ \sum_{\ell, a} T_{\ell, a}\right] \le 7T - 1
$$
By a very similar argument we can also establish a similar bound for MDP $(\ell,a)$.

%\deb{What follows is a quick sketch of a proof, need to formalize with concentration inequalities.}

\textbf{Lower Bound on Regret for Model $P_{\ell,a}$}: In this model the optimal policy is to navigate to the $\ell$-th leaf node and then take action $a$. Let $W$ be a geometric random variable with parameter $p = \frac{1}{2H}$. Then the value function of an agent $i$ is at least
\begin{align*}
&\left( \frac{1}{2} + \Delta \right) \E\left[ \min\set{W, H-\log S}\right] \\
\end{align*}
We now lower bound the term $\E\left[ \min \set{W, H - \log S}\right]$.
\begin{align*}
\E\left[ \min \set{W, H - \log S}\right]
&= \sum_{x=1}^{H-\log S} x(1-p)^{x-1}\cdot p \\
&\hspace{-15mm}+ \sum_{x\ge H-\log S} (H - \log S) (1-p)^{x-1}\cdot p\\
&= \frac{1}{p} - 2 (1-p)^{H - \log S}\\
&= 2H - 2 \left( 1 - \frac{1}{2H}\right)^{H - \log S}\\
&\ge 2H - 2 \cdot e^{-\frac{1}{2} - \frac{\log S}{2H}} = \Theta(H)
\end{align*}
as long as $H = \Omega(\log S)$.
Therefore, the total expected Nash welfare over the $T$ episodes is at least
\begin{align*}
&\Theta\left( T\left( \frac{1}{2} + \Delta\right)^n H^n\right) \\=& \Theta\left(\left( \frac{1}{2} + \Delta\right)^n H^n \E_{\ell,a}\left[\sum_{\ell',a'} T_{\ell',a'} \right] \right)
\end{align*}
Let $X^t_{\ell',a'}$ be the indicator variable that denotes whether the policy visits leaf node $\ell'$, and takes action $a'$ at episode $t$. If $X^t_{\ell,a} = 1$ then by a very similar argument as above, we can show that expected Nash welfare at episode $t$ is at most $O((1/2 + \Delta)^n H^n)$. On the other hand, if $X^t_{\ell, a} = 0$ then expected Nash welfare at episode $t$ is $O((1/2)^n H^n)$. Therefore, the sum of expected Nash welfare over the $T$ episodes is
\begin{align*}
&\sum_t \sum_{(\ell', a') \neq (\ell, a)} \E_{\ell,a}\left[X^t_{\ell',a'} \right] \left(\frac{H}{2} \right)^n\\
&+ \sum_t  \E_{\ell,a}\left[X^t_{\ell,a} \right] \left(\frac{1}{2} + \Delta \right)^n H^n\\
&= \left( \frac{H}{2}\right)^n \sum_{\ell', a'} \E_{\ell,a}\left[T_{\ell',a'} \right] \\&+ \E_{\ell,a}\left[T_{\ell,a} \right] \left( \left(\frac{1}{2} + \Delta \right)^n - \left(\frac{1}{2} \right)^n\right)H^n\\
\end{align*}
Let $T_\sigma = \sum_{\ell',a'} T_{\ell',a'}$. Then regret on model $P_{\ell,a}$ is at least
\begin{align*}
 R_{\ell,a} &\ge \left( \left(\frac{1}{2} + \Delta \right)^n - \left(\frac{1}{2} \right)^n\right)H^n \E_{\ell,a}\left[T_\sigma - T_{\ell,a} \right] \\
 &\ge n \Delta \left( \frac{H}{2}\right)^n  \E_{\ell,a}\left[T_\sigma - T_{\ell,a} \right]
\end{align*}
Now we can proceed very similarly to the proof of theorem~38.7 from \cite{LS20} and establish that there exists some $\ell',a'$ so that $\E_{\ell',a'}[T_\sigma - T_{\ell',a'}] \ge O(T)$ for $\Delta = \Omega(\sqrt{SA/T})$. This choice of $\Delta$ implies the following lower bound.
$$
R_{\ell',a'} \ge \Omega \left(n \left(\frac{H}{2} \right)^n \sqrt{SAT} \right)
\qedhere
$$
\end{proof}

\section{Proof of Theorem~\ref{thm:bound-new-regret}}
\begin{proof}
 We first assume that $v^\star$ equals the maximin value $v^\star_\MW = \max_\pi \min_i V^\pi(r_i)$. We will assume $B \ge 1$. Notice that UOB-REPS is run with reward function 
 $$
 \tilde{r}_t(s,a) = \sum_i \lambda_i^t r_i(s,a) + \frac{v^\star}{H}\left( 1 - \sum_i \lambda_i^t\right)
 $$
 at time $t$. Since each entry of the reward function $r_i$ is bounded by $1$ we have $\abs{\tilde{r}_i(s,a)} = O(B)$. Therefore, from the regret guarantee of UOB-REPS we have 
\begin{align}
&\max_q \sum_{t=1}^T \calL(q,\lambda^t)  - \sum_{t=1}^T \calL(q^t, \lambda^t) \le \tilde{O}\left( BHS\sqrt{AT}\right)\nonumber \\
\Rightarrow &\max_q T \cdot \calL(q, \bar{\lambda}) - \sum_{t=1}^T \calL(q^t, \lambda^t) \le \tilde{O}\left( BHS\sqrt{AT}\right) \label{eq:temp-regret-1}
\end{align}
where $\bar{\lambda} = \frac{1}{T} \sum_{t=1}^T \lambda_t$.
On the other hand, from the regret guarantee of OSMD we have,
\begin{align}
    &\sum_{t=1}^T \calL(q^t, \lambda^t) - \min_{\lambda \in \calC} \sum_{t=1}^T \calL(q^t, \lambda)   \le {O}\left(BH\sqrt{nT\log A} \right)\nonumber \\
    \Rightarrow &\sum_{t=1}^T \calL(q^t, \lambda^t) - \min_{\lambda \in \calC} T \cdot \calL(\bar{q}, \lambda) \le O\left( BH \sqrt{n T \log A}\right) \label{eq:temp-regret-2}
\end{align} 
where $\bar{q}$ is defined as $\bar{q}_h(s,a) = \frac{1}{T} \sum_{t=1}^T q^t_h(s,a)$.
Let $i^\star \in \argmin_{i \in [n]} \sum_{t=1}^T V^{\pi_t}(r_i)$. Since $\E[V^{\pi_t}(r_i)] = \sum_h \sum_{s,a} q_h^t(s,a) r_i(s,a)$, an application of Chernoff's implies the following inequality holds with probability at least $1-\delta$.
\begin{align*}
&\abs{\sum_{t=1}^T \sum_h \sum_{s,a} q^t_{h}(s,a) r_{i^\star}(s,a) - \sum_{t=1}^T V^{\pi^t}(r_{i^\star})} \\
&\le O\left( H \sqrt{T\log(n/\delta)}\right)
\end{align*}
Let $e_{i^\star}$ be the unit vector with exactly $1$ at index $i^\star$, and $0$ otherwise. Then from \cref{eq:temp-regret-2} we get the following bound.
\begin{align*}
    &T\cdot \calL(\bar{q}, B e_{i^*}) \ge \min_{\lambda \in \calC} T \cdot \calL(\bar{q}, \lambda) \\&\ge \sum_{t=1}^T \calL(q^t, \lambda^t) - O\left( BH\sqrt{n T \log A}\right)\\
    &\ge \max_q T \cdot \calL(q,\bar{\lambda}) \\&- \tilde{O}\left(BHS\sqrt{AT}\right) - O\left( BH\sqrt{nT\log A}\right)\ \textrm{[By~\eqref{eq:temp-regret-1}]}\\
    &= T \cdot v^\star \left(1-\sum_i \bar{\lambda}_i \right) + T \cdot \max_q \sum_i \bar{\lambda}_i \sum_{h} \sum_{s,a} q_h(s,a)  \\ &- \tilde{O}\left(BHS\sqrt{AT}\right) - O\left( BH\sqrt{nT\log A}\right)\\
    &\ge T \cdot v^\star \left(1-\sum_i \bar{\lambda}_i \right) + T \cdot v^\star \sum_i \bar{\lambda}_i \\&- \tilde{O}\left(BHS\sqrt{AT}\right) - O\left( BH\sqrt{nT\log A}\right)\\
    &= T \cdot v^\star - \tilde{O}\left(BHS\sqrt{AT}\right) - O\left( BH\sqrt{nT\log A}\right)
\end{align*}
We can upper bound $T\cdot \calL(\bar{q}, B e_{i^\star})$ as
\begin{align}
    &T \cdot \left( v^\star + B \left(\sum_{h=1}^H \sum_{s,a} \bar{q}_h(s,a) r_{i^\star}(s,a) - v^\star \right) \right) \nonumber \\&= T v^\star (1 - B) + B \sum_{t=1}^T \sum_{h=1}^H \sum_{s,a} q_h^t(s,a) r_{i^\star}(s,a)\label{eq:ubd-lhs}
\end{align}
This equality gives us the following bound.
\begin{align*}
    &- Tv^\star B + B \min_i \sum_{t=1}^T V^{\pi_t}(r_i) \ge \\&-\tilde{O}\left(BHS\sqrt{AT}\right) - O\left( BH\sqrt{nT\log A}\right)
\end{align*}
After rearranging, and dividing throughout by $B$ we get the following bound on regret.
$$
Tv^\star - \min_i \sum_{t=1}^T V^{\pi_t}(r_i) \le \tilde{O}\left(HS\sqrt{AT}\right) + O\left( H\sqrt{nT\log A}\right)
$$

Now consider the case when $v^\star \neq v^\star_\MW$. Pick any $v^\star \ge v^\star_\MW$ and run algorithm~\ref{algo:lagrange-maximin}. By a similar argument as above, we can establish the following bound.
\begin{align*}
&T \cdot \calL(\bar{q}, B e_{i^\star}) \ge  T \cdot v^\star \left(1-\sum_i \bar{\lambda}_i \right) + T \cdot v^\star_\MW \sum_i \bar{\lambda}_i\\ & - \tilde{O}\left(BHS\sqrt{AT}\right) - O\left( BH\sqrt{nT\log A}\right)
\end{align*}
Using the upper bound established in \eqref{eq:ubd-lhs} we get the following bound.
\begin{align*}
    &Tv^\star(1-B) + B \min_i \sum_{t=1}^T V^{\pi_t}(r_i) \ge T \cdot v^\star \left(1-\sum_i \bar{\lambda}_i \right) \\&+ T \cdot v^\star_\MW \sum_i \bar{\lambda}_i - \tilde{O}\left(BHS\sqrt{AT}\right) - O\left( BH\sqrt{nT\log A}\right)
\end{align*}
After rearranging and dividing throughout by $B$ we get the following inequality.
\begin{align*}
    &Tv^\star - \min_i \sum_{t=1}^T V^{\pi_t}(r_i) \le \frac{T}{B} \sum_i \bar{\lambda}_i \left(v^\star - v^\star_\MW \right) \\&+ \tilde{O}\left( HS\sqrt{AT}\right) + {O}\left( H \sqrt{nT\log A}\right)
\end{align*}
This gives us the following bound on regret.
\begin{align*}
    &Tv^\star_\MW - \min_i \sum_{t=1}^T V^{\pi_t}(r_i) \\
    &\le Tv^\star - \min_i \sum_{t=1}^T V^{\pi_t}(r_i) + T(v^\star_\MW - v^\star)\\
    &\le T (v^\star_\MW  - v^\star)\left(1 - \frac{\sum_i \bar{\lambda}_i}{B}\right) \\&+  \tilde{O}\left( HS\sqrt{AT}\right) + {O}\left( H \sqrt{nT\log A}\right)\\
    &\le  \tilde{O}\left( HS\sqrt{AT}\right) + {O}\left( H \sqrt{nT\log A}\right)
\end{align*}
The last inequality follows because $v^\star_\MW \le v^\star$ and $\sum_i \bar{\lambda}_i \le B$.
\end{proof}
\end{document}